\documentclass[11pt]{article}

\usepackage[utf8]{inputenc}
\usepackage[letterpaper]{geometry}
\usepackage{graphicx}
\usepackage{xcolor}
\usepackage{mathtools}
\usepackage{amsmath}
\usepackage{amssymb}
\usepackage{amsthm}
\usepackage{xfrac}
\usepackage{url}
\usepackage{bm}
\usepackage{enumitem}
\usepackage{algorithmic}
\usepackage{ifthen}
\usepackage[font=small]{subfig}
\usepackage{float}
\usepackage[numbers]{natbib} 
\usepackage{blindtext}
\usepackage{multicol}

\usepackage[linesnumbered,ruled,vlined,algo2e]{algorithm2e}

\newtheorem{assumption}{Assumption}
\newtheorem{remark}{Remark}

\title{On the Relevance of Byzantine Robust Optimization\\ Against Data Poisoning}
\author{ Sadegh Farhadkhani$^*$ \and Rachid Guerraoui$^*$  \and Nirupam Gupta$^*$  \and Rafael Pinot$^\dagger$}

\usepackage{amsmath,enumitem,mathtools}
\usepackage{nicefrac}
\usepackage{amsmath}
\usepackage{mathtools}
\usepackage{amsthm}

\usepackage[capitalize,noabbrev]{cleveref}

\def\D{\mathcal{D}}
\def\R{\mathbb{R}}

\def\P{\mathcal{P}}

\def\H{\mathcal{H}}
\def\X{\mathcal{X}}

\newcommand{\expect}[1]{\mathop{{}\mathbb{E}}\left[{#1}\right]}
\newcommand{\condexpect}[2]{\mathbb{E}_{#1}\left[{#2}\right]}

\newcommand{\card}[1]{\left\lvert{#1}\right\rvert}

\providecommand{\iprod}[2]{\ensuremath{\left\langle #1,\,#2  \right\rangle}}
\providecommand{\norm}[1]{\ensuremath{\left\lVert#1\right\rVert }}

\newcommand{\ceil}[1]{\left\lceil{#1}\right\rceil}

\newcommand{\indexvar}[3]{{#3}^{\ifthenelse{\equal{#1}{}}{}{\left({#1}\right)}}_{#2}}

\newcommand{\dev}[1]{\delta_{#1}}

\newcommand{\support}[1]{\textsc{supp}\left({#1}\right)}

\newcommand{\heter}{\zeta}

\newcommand{\SSS}{\mathcal{S}}

\DeclareMathOperator*{\argmin}{arg\,min}

\renewcommand{\paragraph}[1]{\textbf{#1}~}

\newtheorem{theorem}{Theorem}
\newtheorem{lemma}{Lemma}
\newtheorem{corollary}{Corollary}
\newtheorem{theorem**}{\bfseries Theorem}
\newtheorem*{theorem*}{\bfseries Theorem}
\newtheorem*{lemma*}{\bfseries Lemma}

\makeatletter
\newtheorem*{rep@theorem}{\rep@title}
\newcommand{\newreptheorem}[2]{%
\newenvironment{rep#1}[1]{%
 \def\rep@title{#2 \ref{##1}}%
 \begin{rep@theorem}}%
 {\end{rep@theorem}}}
\makeatother
\newreptheorem{theorem}{Theorem}

\makeatletter
\newtheorem*{rep@assumption}{\rep@title}
\newcommand{\newrepassumption}[2]{%
\newenvironment{rep#1}[1]{%
 \def\rep@title{#2 \ref{##1}}%
 \begin{rep@assumption}}%
 {\end{rep@assumption}}}
\makeatother
\newreptheorem{assumption}{Assumption}

\makeatletter
\newtheorem*{rep@lemma}{\rep@title}
\newcommand{\newreplemma}[2]{%
\newenvironment{rep#1}[1]{%
 \def\rep@title{#2 \ref{##1}}%
 \begin{rep@lemma}}%
 {\end{rep@lemma}}}
\makeatother
\newreptheorem{lemma}{Lemma}

\makeatletter
\newtheorem*{rep@proposition}{\rep@title}
\newcommand{\newrepproposition}[2]{%
\newenvironment{rep#1}[1]{%
 \def\rep@title{#2 \ref{##1}}%
 \begin{rep@proposition}}%
 {\end{rep@proposition}}}
\makeatother
\newreptheorem{proposition}{Proposition}

\makeatletter
\newtheorem*{rep@definition}{\rep@title}
\newcommand{\newrepdefinition}[2]{%
\newenvironment{rep#1}[1]{%
 \def\rep@title{#2 \ref{##1}}%
 \begin{rep@definition}}%
 {\end{rep@definition}}}
\makeatother
\newreptheorem{definition}{Definition}

\newcommand{\loss}{Q}

\newcommand{\mmt}[2]{\indexvar{#1}{#2}{m}}
\newcommand{\gradient}[2]{\indexvar{#1}{#2}{g}}
\newcommand{\model}[2]{\indexvar{#1}{#2}{\theta}}
\newcommand{\modelp}{\theta'}
\newcommand{\learningrate}[1]{\gamma_{#1}}
\newcommand{\AvgMmt}[1]{\overline{m}_{#1}}
\newcommand{\avgloss}{\loss^{(\H)}}
\newcommand{\avggrad}[1]{\indexvar{}{#1}{\overline{g}}}
\newcommand{\lyap}[1]{V_{#1}}
\newcommand{\optloss}{\loss^*}

\newcommand{\localloss}[1]{\indexvar{#1}{}{Q}}
\newcommand{\dist}[1]{\indexvar{#1}{}{\mathcal{D}}}

\allowdisplaybreaks

\begin{document}

\newgeometry{left=1in,right=1in,top=1in,bottom=1in} %

\date{}
\maketitle

{\small
\begin{multicols}{2}
\begin{center}
    $^*$IC, EPFL \\
    \texttt{firstname.lastname@epfl.ch}\\
    $^\dagger$Sorbonne Université\\
     \texttt{pinot@lpsm.paris}
\end{center}

\end{multicols}
}

\vspace{8mm}

\begin{abstract}%
The success of machine learning (ML) has been intimately 
linked with the availability of large amounts of data, typically collected from heterogeneous sources and processed on vast networks of computing devices (also called {\em workers}).
Beyond accuracy, the use of ML in critical domains such as healthcare and autonomous driving calls for robustness 
against 
{\em data poisoning}
and 
some 
{\em faulty workers}.
The problem of {\em Byzantine ML} formalizes these robustness issues by considering a distributed ML environment in which workers (storing a portion of the global dataset) can deviate arbitrarily from the prescribed algorithm.
Although the problem has attracted a lot of attention from a theoretical point of view, its practical importance for addressing realistic 
faults (where the behavior of any worker is locally constrained) remains unclear. 
It has been argued that the seemingly weaker threat model where only workers' local datasets get poisoned 
is more reasonable. We prove 
that, while tolerating a wider range of faulty behaviors, Byzantine ML 
yields solutions that are, in a precise sense, optimal even under the weaker data poisoning threat model. Then, we study a generic data poisoning model wherein some workers have {\em fully-poisonous local data}, i.e., their datasets are entirely corruptible, and the remainders have {\em partially-poisonous local data}, i.e., only a fraction of their local datasets is corruptible. 
We prove that Byzantine-robust schemes yield optimal solutions against both these forms of data poisoning, and that the former
is more harmful when workers have {\em heterogeneous} local data.

\end{abstract}

\maketitle

\section{Introduction}
Learning a model using several machines over their collective data is appealing. The motivation behind this {\em distributed} machine learning (ML) scheme (a.k.a.~{\em federated learning}~\citep{KairouzOpenProblemsinFed2021}) is usually efficiency. Another motivation is privacy where each machine retains control over its local data.
The distributed ML problem can be precisely stated as follows in a standard {\em server-based system} comprising $n$ machines (referred as {\em workers}), represented by set $[n] \coloneqq \{1, \dots, \, n\}$, and a server.
Each worker $i$ has access to a common data space $\X$ through a local distribution $\D^{(i)}$. A model parameterized by $\theta \in \R^d$ incurs a loss for each data point $x \in \X$ measured by a real-valued {\em loss function} $q: \R^d \times \X \to \R$. Then, for each worker $i \in [n]$,  the \emph{local loss function} is given by
\begin{align}
\label{eq:local_loss}
    \loss^{(i)}(\model{}{}) \coloneqq \condexpect{x \sim \D^{(i)}}{q(\model{}{}, \, x)}.
\end{align} 
The server aims to compute a model parameter $\theta^* \in \R^d$ minimizing the \emph{global loss function}
\begin{align}
     \loss(\theta) \coloneqq \frac{1}{n} \sum_{i = 1}^n \loss^{(i)}\left( \model{}{} \right) . \label{eqn:global_loss_all}
\end{align}
We assume that the gradient of the loss function $q(\model{}{}, \, x)$ with respect to $\model{}{}$, denoted by $\nabla q(\model{}{}, \, x)$, exists and is continuous at all $\model{}{} \in \R^d$ and $x \in \X$, which is standard in ML~\cite{bottou2018optimization}.

\subsection{Background: Distributed ML with D(S)GD}
Minimizing the global average loss is typically achieved using a first-order distributed method such as the celebrated Distributed Gradient Descent (or DGD) and its {\em stochastic} variant DSGD~\citep{konevcny2016federated}.\footnote{For more details on these distributed methods, refer the book~\cite{bertsekas2015parallel}.} At each iteration $t \geq 0$, the server maintains a model~$\theta_t$, which is broadcast to all the workers. Then, each worker $i$ sends back to the server an update vector that is either their {\em local gradient} $\nabla \loss^{(i)}\left( \model{}{t} \right)$ in the case of DGD or an unbiased stochastic estimate $g_t^{(i)}$ of their local gradient in the case of DSGD.
Finally, the server updates the current model $\model{}{t}$ using the \emph{average} of the local updates sent by the workers. 
When all the workers are {\em honest}, i.e, correctly follow the instructions of the server,
the above iterative procedure provably converges to a parameter $\theta^*$ that is either a minimum or a stationary point of the global loss function depending on whether the function is convex or non-convex, respectively.

\subsection{Threats to Distributed ML}
\label{sec:threat_model}
DSGD (or DGD) is however extremely vulnerable to misbehaving workers that can deviate from the instructions given by the server~\citep{feng2015distributed, su2016fault, krum}. Such misbehavior could result from either inadvertent software/hardware bugs or malicious players controlling part of the system. 
Typically, misbehaving workers are modelled by considering an {\em adversary} that corrupts a fraction of the workers, whose identity is a piori unknown~\cite{guerraoui2023byzantine}. The corruptions induced by the adversary can be characterized by two threat models:
{\em Byzantine failure} (a.k.a. model poisoning) and {\em data poisoning}~\cite{shejwalkar2021manipulating, equivalenceposiong}. 
\begin{enumerate}
\setlength{\itemsep}{0.5em}
    \item \textbf{Byzantine failure.}
    In this particular threat model, we assume that a corrupted worker can deviate arbitrarily 
    from its prescribed algorithm~\citep{lamport1982byzantine}. In the context of DSGD, a Byzantine worker can send (arbitrary) malicious vectors for its local gradients to the server~\citep{little, empire}. 

    \item \textbf{Data poisoning.} In this particular threat model, we assume that a corrupted worker follows the prescribed algorithm correctly but its local dataset can be poisoned~\citep{MahloujifarMM19}. In the context of DSGD, while the gradients sent by a worker $i$ need not be arbitrary, they can correspond to a data distribution $\widetilde{\D}^{(i)}$ that differs from the true data distribution $\D^{(i)}$.
\end{enumerate}
    
Note that the former, i.e., the Byzantine failure threat model, subsumes the latter, i.e., the data poisoning threat model. Nevertheless, the latter has received more attention in the past mainly due to its relevance even in the conventional {\em centralized} ML~\citep{charikar2017learning,diakonikolas2019sever,prasad2020robust}. Although the defenses proposed for data poisoning can be extended to Byzantine threat model in distributed ML, e.g., see~\citep{chen2017distributed, yin2018byzantine}, they rely upon {\em data homogeneity}, i.e., the honest workers are assumed to have identical local data distributions~\citep{dia22}. In general distributed ML however the workers have {\em heterogeneous data}, i.e., their local data distributions are distinct~\citep{collaborativeElMhamdi21, data2021byzantine_icml, karimireddy2022byzantine, farhadkhani2023robust}.
The data poisoning threat can be further classified into two cases: {\em fully-poisonous local data} and {\it partially-poisonous local data}. Suppose that worker $i$ is corrupted by an adversary. In the case of fully-poisonous local data,  
the entire local dataset of worker $i$ can be poisoned, i.e., $\widetilde{\D}^{(i)}$ is truly arbitrary. In the case of partially-poisonous local data, only a fraction (of unknown identity) of worker $i$'s local dataset is corrupted. These two forms of data poisoning in distributed ML were introduced in~\citep{MahloujifarMM19}.

\subsection{Byzantine failure vs data poisoning}
Given the heterogeneous nature of workers' data in distributed ML, it seems reasonable to seek novel solutions to data poisoning. But what about Byzantine failures? One could argue that a truly arbitrary behavior is largely fictitious and unlikely to be realized in practice~\citep{shejwalkar2022back}. Indeed, each worker of a distributed system is typically restricted to very limited local information and cannot possibly be omniscient, unlike what is assumed in the Byzantine threat model~\cite{karimireddy2022byzantine, farhadkhani2022byzantine}. Somehow, the cost of defending against Byzantine workers might not be justifiable compared to the cost for defending against data poisoning. But what is that cost difference anyway? The motivation of this work is to address that question, and equivalently, the following question: 
\begin{center}
    \it Is defending against Byzantine failure an overkill with respect to data poisoning?
\end{center}
We answer this question negatively in the context of a large class of ML problems. 
We prove (perhaps surprisingly) that, although the Byzantine failure threat model is strictly stronger, the best learning guarantees that a first-order distributed algorithm, such as DSGD, can achieve under this threat are optimal even in the weaker data poisoning threat model. Furthermore, we precisely characterize the impact on the learning due to both full-poisonous and partially-poisonous local data. We show that in real-world applications when workers' datasets are heterogeneous, see~\cite{KairouzOpenProblemsinFed2021}, fully-poisonous local data is a stronger adversarial setting. Our contributions are summarized in the following.

\subsection{Main results}
\label{sec:mainres}
\paragraph{Solution to Byzantine failure is tight with respect to data poisoning.} We consider the class of ML problems that can be solved by optimizing $L$-Lipschitz smooth loss functions satisfying the $\mu$-PL inequality, where the local gradients (of honest workers)
have bounded covariance trace of $\sigma^2$. These conditions are satisfied in many cases~\cite{bottou2018optimization}. We further assume that the global {\em gradient dissimilarity} that characterizes data heterogeneity is bounded by $\heter^2$, which is essential to tackling misbehaving workers of either type~\cite{karimireddy2022byzantine, allouah2023fixing}. We assume that the total number of fully corrupted workers is bounded by $f$. Note that the case of $f \geq n/2$ is trivial as the learning error can be arbitrarily large  (Lemma 1 in~\citep{liu2021approximate}): we thus assume $f < n/2$ in all our results.

\begin{enumerate}[leftmargin = 1.7em]
\setlength{\itemsep}{0.5em}

    \item {\bf Lower bound under data poisoning.} We first characterize the suboptimality gap (or error) of a stochastic first-order distributed algorithm under data poisoning. Specifically, we show that with $f$ workers with corrupted data the error is in $\Omega \left( \textcolor{black}{\frac{f}{n} \cdot \frac{\heter^2}{\mu}} \right)$. Moreover, the convergence rate (a.k.a.~iteration complexity) to realize an $\varepsilon$-approximation of this error is in
    \begin{equation}
    \label{eqn:lb}
        \Omega \left( \textcolor{black!100}{ \frac{1 + f}{n} \cdot \frac{\sigma^2}{\mu \varepsilon} } + \textcolor{black}{\frac{L}{\mu} \cdot \log\frac{Q_0}{\varepsilon}}\right)\enspace,
    \end{equation}
    where $Q_0$ is the initial error of the algorithm.
    These lower bounds characterize how {\em good} and {\em fast} we can learn using $n$ workers when $f$ of the workers suffer from local data poisoning.

    \item {\bf Matching upper bound under Byzantine failure.} We then consider the Byzantine-robust adaptation of DSGD, incorporating distributed {\em Polyak's momentum} and {\em coordinate-wise trimmed mean} from~\cite{farhadkhani2022byzantine}. We show that, despite the presence of $f$ Byzantine corrupted workers, this algorithm achieves an
    error in $ \mathcal{O} \left( \textcolor{black}{\frac{f}{n} \cdot \frac{\heter^2}{\mu}} + \varepsilon\right)$ with a convergence rate of 
    \begin{equation}
        \mathcal{O}  \left( \textcolor{black!100}{ \frac{1 + f}{n} \cdot \frac{K\sigma^2}{\mu \varepsilon} } + \textcolor{black}{\frac{L}{\mu} \cdot \log\frac{Q_0}{\varepsilon}} \right)\enspace, \label{eqn:intro_ub}
    \end{equation}
    where $K \coloneqq \frac{L}{\mu}$ is the \emph{condition number} of the average loss function for the honest workers. Hence, when $K \in \mathcal{O}(1)$, we get a matching upper bound to the lower bound in the data poisoning threat (which automatically also applies to the Byzantine failure threat). To the best of our knowledge, this is the first tight analysis of Byzantine robustness in terms of the convergence rate of a first-order method.  
    The state-of-the-art result in~\cite{allouah2023privacy} features a {\em sublinear} convergence rate in $\mathcal{O}\left(\frac{1}{\sqrt{\varepsilon}}\right)$ even when honest workers compute exact local gradients, i.e., $\sigma = 0$.\\
\end{enumerate}

\paragraph{Partially-poisonous vs fully-poisonous local data.} We then consider a scenario where in addition to having $f$ out of $n$ workers with fully-poisonous local datasets, each worker can have partially-poisonous local data. Specifically, we assume that each worker $i$ has $b$ number of corruptible data points out of $m$ total data points. Note that in this particular case, for each worker $i$ the distribution $\D^{(i)}$ is given by the uniform distribution over the $m - b$ incorruptible local data points. We prove that the optimization error is in
\begin{align*}
    \Theta\left(\frac{f}{n}\cdot\frac{\zeta^2}{\mu} + \frac{b}{m} \cdot\frac{\sigma^2}{\mu} \right) \enspace .
\end{align*}
We show that the above error, which is optimal in general, can be achieved using a {\em Byzantine-robust} first-order method with an exponential convergence rate (i.e., logarithmic iteration complexity). Hence, demonstrating the tightness of Byzantine-robust schemes even against data poisoning at the local level. Moreover, as the error resulting from partially-poisonous local data is independent of the heterogeneity factor $\zeta$, this result also shows that in practical distributed ML applications, where dataset heterogeneity among workers is often significant~(Karimireddy et al., 2020), fully-poisonous local data alone (i.e., $\frac{f}{n} = \delta > 0$, and $b = 0$) is a stronger adversarial setting than partially-poisonous local data alone (i.e., $\frac{b}{m} = \delta > 0$, and $f = 0$), when considering the same fraction $\delta$ of corrupted data points in the system.

\subsection{Key elements of our proof}

Our proof for the lower bound in the homogeneous case, i.e., the first term in~\eqref{eqn:lb}, involves an extension of Huber's general contamination model~\citep{huber64,dia22}. Specifically, we consider a special distributed ML problem of mean estimation where each 
worker samples data points from a common distribution $\D$, and the goal for the server is to compute the true mean of $\D$ in the case when $f$ out of $n$ workers can sample data points from arbitrary distributions. We show that solving this problem using a robust implementation of DSGD with $T$ iterations reduces to robust mean estimation using $n$ {\em batches} of $T$ i.i.d.~data points from $\D$ with $f$ batches being arbitrarily corrupted. 
To derive the lower bound due to heterogeneity, we consider the mean estimation problems with workers sampling data points from two distinct Dirac delta distributions with means $\frac{\heter}{\mu} \sqrt{\frac{n}{f}}$ apart.
We conclude the result by considering two indistinguishable executions, exploiting the anonymity of corrupted workers. Details can be found in Section~\ref{sec:lower_bnd}.
 The more challenging part of our analysis lies in proving a tight upper bound in the Byzantine setting.
 To prove the matching upper bound, we consider a Byzantine-robust adaptation of DSGD, originally proposed in~\cite{farhadkhani2022byzantine}, that uses Polyak's momentum operation at workers' end
 and replaces the averaging at the server by coordinate-wise trimmed mean. Although this algorithm has been shown to guarantee a tight asymptotic error under Byzantine threat model~\citep{allouah2023fixing}, its convergence rate
 remained loose for the specific class of PL functions that we consider (cf.~\cite{allouah2023privacy, data2021byzantine_icml}). To overcome the shortcoming, we consider a scheduled diminishing step sizes (or learning rates), generalizing the results on the tightness of SGD~\citep{stich19,khaled20}.
 The caveat of varying step sizes however is that it leads to {\em dynamic} momentum coefficient, if we are to obtain a tight convergence rate in the presence of Byzantine failures. This renders the existing proof techniques for analyzing the convergence of this particular class of algorithms inapplicable (see~\cite{Karimireddy2021, farhadkhani2022byzantine, allouah2023fixing}), mainly because we can no longer obtain a uniform bound on the {\em momentum drifts}. While~\cite{allouah2023privacy} addresses this challenge using a {\em time-variant} Lyapunov function (see~\cite[Appendix D.2.1]{allouah2023privacy}), the resulting convergence analysis is loose in the precise sense that it features a sublinear convergence rate, which we mentioned above in Section~\ref{sec:mainres}.
 To remedy this, we design a novel {\em time-invariant} Lyapunov function that includes an additive term of appropriately scaled momentum drift (see Section~\ref{sec:intuition}). %

\subsection{Conjecture on the tightness of the upper bound}

Our upper bound (in~\eqref{eqn:intro_ub}) holds for any smooth PL loss function.
Our lower bound (in~\eqref{eqn:lb}) however is derived by considering a quadratic loss function
that is strongly convex\footnote{Strong convex functions constitute a subclass of PL functions.} with condition number $K =1$, which renders 
our overall analysis loose in terms of $K$. We however conjecture our upper bound to be tight (even in the condition number) for the class of loss functions we consider. Indeed, if we assume $f = 0$, our upper bound matches the best known result for the class of smooth PL loss functions~\citep{karimi16}. Moreover, while we are not aware of any lower bound in stochastic optimization that is specific to the PL functions, it was recently shown in~\citep{PL-lower} that, in the non-stochastic case, the dependence of the lower bound on the condition number is indeed different for strongly convex and PL functions. Accordingly, we believe that obtaining a tight result in terms of the condition number $K$ would involve demonstrating that, for general PL loss functions, the convergence rate of a stochastic first-order method is in $\Omega \left( \textcolor{black!100}{ \frac{1 + f}{n} \cdot \frac{K\sigma^2}{\mu \varepsilon} } + \textcolor{black}{\frac{L}{\mu} \cdot \log\frac{Q_0}{\varepsilon}} \right).$               

\subsection{Other related work}
Prior work on Byzantine ML with tight asymptotic error guarantees, relying on either Polyak's momentum or variance-reduction schemes, include~\cite{karimireddy2022byzantine, allouah2023fixing, gorbunov2023variance}. These papers however do not provide tight analysis on the convergence rate for the class of PL functions (or even strongly convex functions) that we consider. The tightest existing result provided in~\cite{allouah2023privacy} features a sublinear convergence rate even in the absence of any stochasticity, compared to the optimal linear convergence rate that we prove. Moreover, many of these results rely on 
constant step sizes (i.e., learning rates) and momentum coefficients, which yield a \emph{uniform bound} on the drift between the local momentums (e.g., Lemma 1 in \cite{farhadkhani2022byzantine}, Lemma 8 in \cite{karimireddy2022byzantine}, and Lemma 6 in \cite{allouah2023fixing}). However, obtaining a tight convergence rate for PL functions calls for diminishing step sizes~\cite{stich19,khaled20}. As the momentum coefficients are coupled with the step sizes, for the sake of Byzantine-robustness, diminishing step sizes result in a dynamic momentum coefficients. Accordingly, we can only obtain a recursive bound on the momentum drift, which makes the analysis more intricate.

Another work that provides a comparison between the Byzantine failure and the data poisoning threats in distributed ML includes~\cite{alistarh2018byzantine}. However, there are several notable distinctions. First,~\cite{alistarh2018byzantine} considers the i.i.d.~case where all honest workers sample data points from the same distribution. Second, the lower and upper bounds in~\cite{alistarh2018byzantine} are not obtained under exactly the same assumptions. The lower bound (Theorem 5.5 in~\cite{alistarh2018byzantine}) is derived by considering a Gaussian data distribution, whereas the upper bound relies on the assumption that the distribution of the stochastic gradients has a uniformly bounded support, 
which is not the case for a Gaussian distribution. We remark that the bounded-support assumption considerably weakens the Byzantine failure threat model as it ensures that the pairwise distances between honest local gradients are bounded. Until now, it remained unclear whether a tight upper bound could be obtained without restricting the Byzantine adversary, and under standard learning assumptions.  

\subsection{Paper organization}
Section~\ref{sec:ModelSetting} presents the problem  
statement. 
Section~\ref{sec:lower_bnd} presents the lower bound under the (fully-poisonous) data poisoning threat model. Section~\ref{sec:upper_bound} presents the matching upper bound under Byzantine failure. Section~\ref{sec:intuition} presents an outline of our upper bound proof, specifically the analysis of the algorithm. Section~\ref{sec:partially_poisonous} introduces the case of partially-poisonous local data and compare it with the fully-poisonous case. Section~\ref{sec:remarks} provides concluding remarks and a discussion on open problems. Detailed proofs are deferred to appendices~\ref{app:lower_bnd} and~\ref{appendix:upperBound}.

\section{Problem Statement and Assumptions}  
\label{sec:ModelSetting}
We consider a server-based system architecture with $n$ workers and a central server. The workers only communicate with the server and there is no communication between workers. We assume that at most $f$ out of $n$ workers may be faulty, either as per Byzantine failure or fully-poisonous local data. We denote by $\H$ the set of $n-f$ honest workers, and let $\loss^{(\H)} (\model{}{})$ denote their average loss, i.e., 
\begin{equation}
    \loss^{(\H)} (\model{}{}) = \frac{1}{\card{\H}} \sum_{i \in \H} \loss^{(i)}\left( \model{}{} \right), \quad \forall \model{}{} \in \R^d \enspace. \label{eqn:global_loss}
\end{equation}
We assume that $\loss^{(\H)}$ admits a minimum, i.e., $\exists \theta^* \in \R^d$ such that for all $\theta \in \R^d$, $\loss^{(\H)}(\theta) \geq \loss^{(\H)}(\theta^*)$. We let $Q^* \coloneqq \loss^{(\H)}(\theta^*)$. Furthermore, we consider the class of smooth loss functions satisfying the Polyak-\L{}ojasiewicz (PL) inequality, which is more general than strong convexity~\citep{bottou2018optimization} and can indeed be satisfied by some non-convex functions~\citep{karimi16}. 
\begin{assumption}[Smoothness]
\label{asp:lip}
    There exists $L < \infty$ such that for all $i \in [n]$ and $\model{}{}, \modelp \in \R^d$, 
    \begin{equation*}
        \norm{\nabla\localloss{i}(\model{}{}) - \nabla\localloss{i}(\modelp)} \leq  L \norm{\modelp-\model{}{}}\enspace.
    \end{equation*}
\end{assumption}
\begin{assumption}[PL-condition]
    \label{asp:polyak} There exists $ \mu \geq 0$ such that for all $\model{}{}\in \R^d$, 
     \begin{align*}
         \norm{\nabla \avgloss(\model{}{})}^2 \geq 2 \mu \left(\avgloss(\model{}{})- \optloss \right) \enspace.
     \end{align*}
\end{assumption}

As stated below, we also assume that the stochastic gradients computed by the honest workers have a bounded local covariance trace. This assumption is standard for analyzing the convergence of stochastic first-order methods~\citep{TangLYZL18}. For all $i \in \H$, by definition of $\localloss{i}$, and the assumption that $ \nabla q \left(\model{}{},x\right)$ is continuous in and $\model{}{}$ and $x$, we have $\condexpect{x \sim \dist{i}}{\nabla q\left(\model{}{},x\right)} = \nabla \localloss{i}(\model{}{})$.
\begin{assumption}[Stochasticity] 
\label{asp:bnd_var}
There exists $\sigma < \infty$ such that for all $i \in \H$ and $\model{}{}\in \R^d$,
\begin{align*}
    \condexpect{x \sim \dist{i}}{\norm{\nabla q \left(\model{}{},x\right)-\nabla \localloss{i}(\model{}{})}^2} \leq \sigma^2 \enspace.   
\end{align*}
\end{assumption}
Lastly, as stated below, we assume the local gradients of the honest workers to have bounded diversity (or {\em heterogeneity}) over the parameter space. Without this assumption we cannot obtain meaningful guarantees in the threat models we consider, as shown in~\citep{karimireddy2022byzantine}. 

\begin{assumption}[Heterogeneity]
    \label{asp:heter}
    There exists $\heter < \infty$ such that for all $\model{}{}\in \R^d$, 
    \begin{equation*}
        \frac{1}{\card{\H}}\sum_{i \in \H} {\norm{\nabla\localloss{i}(\model{}{})-\nabla\avgloss(\model{}{})}^2} \leq \heter^2 \enspace.
    \end{equation*}
\end{assumption}

\section{Lower Bound with Data Poisoning (fully-poisonous local data)}
\label{sec:lower_bnd}

We characterize here the limitation of iterative stochastic first-order distributed algorithms in the data poisoning model. Specifically, we consider a generic randomized distributed algorithm $\Pi$ that executes in $T$ iterations. We define an execution of $\Pi$ as follows. The server begins by choosing an initial parameter vector $\model{}{0}$. In each iteration $t \geq 0$, the server maintains a parameter vector $\model{}{t} \in \R^d$ that is broadcast to the workers. Each honest worker $i$ then samples one data point $x^{(i)}_t$ from its local distribution $\D^{(i)}$, computes a gradient $g^{(i)}_t = \nabla q\left( \model{}{t}, x^{(i)}_t \right)$, and sends back to the server a message
$$\text{msg}^{(i)}_t = \Psi_t\left( (\model{}{\tau})_{0 \leq \tau \leq t}, \,  (g^{(i)}_\tau)_{0 \leq \tau \leq t}  \right) \enspace,$$ 
where $\Psi_t: \R^{d \times t} \times \R^{d \times t} \to \R^d $. A faulty worker $j$ with a fully-poisonous dataset behaves exactly like an honest worker, except it samples its data point from an arbitrary distribution $\widetilde{\D}^{(j)}$ instead of its true local distribution $\D^{(j)}$. The server then proceeds to update the current parameter vector $\model{}{t}$ to $\model{}{t+1}$. At the completion of the $T$-th iteration, the server outputs $\hat{\model{}{}}$. Note that this generic formulation includes the class of first-order optimization methods such as D-SGD and distributed momentum~\cite{POLYAK19641, farhadkhani2022byzantine}.  
We obtain a lower bound on the sub-optimality of $\Pi$, presented in Theorem~\ref{thm:lower_bound_strongly_convex}, when there are at most $f < n/2$ faulty workers. The lower-bound is agnostic to the functions $\{\Psi_t\}_{t \in \{0,\dots, T-1\}}$ that the workers implement to generate their messages, or the methods that the server implements to update its parameter vectors and generate the output. 

\begin{theorem}
    \label{thm:lower_bound_strongly_convex}
    Suppose assumptions~\ref{asp:lip},~\ref{asp:polyak},~\ref{asp:bnd_var}, and~\ref{asp:heter} hold true. Let $Q_0 := \avgloss \left({\model{0}{}} \right)- Q^* $. 
    Consider algorithm $\Pi$ as described above. %
    If there exists $A \geq 0$ such that
    $\condexpect{\Pi}{\avgloss \left( \hat{\model{}{}} \right) - Q^*} \leq A$, then %
    \begin{align*}
        A \in \Omega\left(\frac{f}{n} \cdot \frac{\heter^2}{\mu}\right) \enspace,
    \end{align*}
     where $\condexpect{\Pi}{\cdot}$ denotes the expectation over the randomness in $\Pi$.  Moreover, 
     we can guarantee that $\condexpect{\Pi}{\avgloss \left( \hat{\model{}{}} \right) - Q^*} \in \mathcal{O}\left(\frac{f}{n} \cdot \frac{\heter^2}{\mu} + \varepsilon\right)$ only if 
     \begin{align*}
         T \in \Omega \left( \textcolor{black!100}{ \frac{1 + f}{n} \cdot \frac{\sigma^2}{\mu \varepsilon} } + \textcolor{black}{\frac{L}{\mu} \cdot \log\frac{Q_0}{\varepsilon}} \right).
     \end{align*}
\end{theorem}
\begin{proof}[Proof sketch] We present here a sketch of our proof, and defer the formal proof to Appendix~\ref{app:lower_bnd}.
We prove the theorem for the scalar domain, i.e., $d = 1$, $\X \in \R$, and a quadratic loss, i.e., $q(\model{}{}, x) = \frac{\mu}{4} \left( \model{}{} - x \right)^2$. 
As the lower bound is established
using the squared Euclidean norm, the proof applies directly to $d > 1$ since the instances used in the proof are still valid in a 1-dimensional subspace.
We consider two separate cases, where the first case obtains the non-vanishing error term and the second case lower bounds the convergence rate.

{\bf First case.} In this case, using the idea in Theorem III in \citep{karimireddy2022byzantine} we derive a lower bound on the error when honest workers may have non-identical data distributions, which is the non-vanishing error term in Theorem~\ref{thm:lower_bound_strongly_convex}. We partition the set of workers into $S = \{1, \ldots, \, n-f\}$ and $\hat{S} = \{n-f+1, \ldots, \, n\}$, and consider the following Dirac distributions:
    \begin{align*}
        \text{Distribution} \quad \D^{(i)} : \begin{cases}
          x = 0  ~ \text{, w.p. } ~ 1 ~ ; & i \in S\\
          x = \frac{2 \heter}{\mu}\sqrt{\frac{n-f}{f}} ~ \text{, w.p. } ~ 1 ~ ; & i \in \hat{S}
      \end{cases} 
    \end{align*}
    We consider two valid executions of $\Pi$ with different identities for the honest workers. In Execution 1, $\H = S$ and in Execution 2, $\H = \{1, \ldots, \, n-2f\} \cup \hat{S}$. As the guarantee of algorithm $\Pi$ must hold true in both these executions, upon simply applying the condition on the loss function $\avgloss (\hat{\model{}{}} )$ in both executions, we conclude that $ \varepsilon \in \Omega\left(\nicefrac{f}{n} \cdot 
    \nicefrac{\heter^2}{\mu}\right).$

{\bf Second case.} In this case, we consider homogeneity, i.e., let $\D^{(i)} = \D$ for all $i \in \H$. Recall that in each execution of $\Pi$ each worker computes a batch of $T$ stochastic gradients, and $f$ out of these $n$ batches  may be corrupted. Thus, upon extending the Huber's contamination model (see e.g.~\citep{dia22}) to batch sampling, we can show that it is impossible for $\Pi$ to tell whether the honest workers send stochastic gradients corresponding to distribution $\D$ or another distribution $\D'$, both satisfying Assumption~\ref{asp:bnd_var}, if $\text{TV}\left( \D^T, \, \D'^T \right) \leq \frac{2f}{n}$.\footnote{$\text{TV}$ represents the {\em total variation distance} between two probability measures~\citep{gibbs2002choosing}.} We realize this scenario by the following instances:
\begin{align*}
        \text{Distribution } ~ \D: & \quad x = 0 ~ ,\text{ \quad w.p.} ~ 1.\\
        \text{Distribution } ~ \D': & \quad x =
        \begin{cases}
          \frac{2\sigma}{\mu} \sqrt{\frac{Tn}{2f}}  & ~ , \text{ \quad w.p.} ~ \frac{2f}{nT} \\
          0 & ~ , \text{ \quad w.p. } ~ 1 - \frac{2f}{nT} 
      \end{cases}
    \end{align*}
As $\left(\condexpect{x \sim \D'}{x} -  \condexpect{x \sim \D}{x}\right)^2 = \left(\frac{2\sigma}{\mu} \sqrt{\frac{2f}{nT}} - 0 \right)^2 = \frac{f}{n} \cdot \frac{4 \sigma^2}{\mu^2 T}$, for the considered quadratic loss function $q(\model{}{}, x) = \frac{\mu}{4} \left( \model{}{} - x\right)^2$, we conclude that 
\begin{align*}
    \condexpect{\Pi}{\avgloss \left( \hat{\model{}{}} \right) - Q^*} \in \Omega\left( \frac{f}{n} \cdot \frac{\sigma^2}{\mu T} + \frac{1}{n} \cdot \frac{\sigma^2}{\mu T}\right) \enspace,
\end{align*}
which means to get an $\varepsilon$-approximate solution, we must have
\begin{align*}
    T \in \Omega\left( \frac{f+ 1}{n} \cdot \frac{\sigma^2}{\mu \varepsilon}\right).
\end{align*}
While the first term in the argument of $\Omega$ above comes from the fact that we cannot distinguish between the two valid distributions $\D$ and $\D'$, the second term is due to the classical lower bound on the minimax statistical error considering Gaussian distributions~\citep{wu2017lecture}, i.e., the worst-case squared-error incurred in estimating the mean of a distribution with variance $\sigma^2$ from at most $nT$ i.i.d.~samples. Finally, in the case where $\sigma = 0$ and $\zeta = 0$, i.e., all the honest workers send the same gradient vector, we have the lower bound of $\Omega(\nicefrac{L}{\mu} \cdot \log\nicefrac{Q_0}{\varepsilon})$, shown in~\cite{yue2023lower}. 

We conclude by composing the bounds obtained in the different cases.
\end{proof}

\section{Upper Bound with Byzantine failure}
\label{sec:upper_bound}

We present here a matching upper bound for Theorem~\ref{thm:lower_bound_strongly_convex},  considering a Byzantine adversary. We first describe the algorithm  we consider, and then present its convergence guarantee.

\subsection{Algorithm Description}
\label{sec:algo}

\begin{algorithm2e}[!htb]
    
    \SetKwInOut{Input}{Input}
    \SetKwInOut{Output}{Output}
    \SetAlgoLined
    \Input{$T \geq 2$, $\left(\learningrate{0},\dots,\learningrate{T-1}\right)$ and $\left(\beta_{0},\dots,\beta_{T-1}\right)$.}  
    \textcolor{violet}{\bf Server} chooses arbitrarily $\theta_0 \in \R^d$. Each \textcolor{black}{\bf honest worker} $i$ sets $m_{-1}^{(i)} = 0$.
    
    \For{$t=0$ to $T-1$}{
    \textcolor{violet}{\bf Server} broadcasts $\model{}{t}$ to all workers\;
    
    \For{each \textcolor{teal}{\bf honest worker} $i$ (in parallel)}{
        Compute a stochastic gradient $\gradient{i}{t}$, as defined in \eqref{eqn:localgradient};

        Send to the server the momentum $\mmt{i}{t}$, as defined in~\eqref{eqn:mmt_i};
    }
    
    \textcolor{darkgray} { \% A corrupted worker $i$ 
    may send an arbitrary value for $\mmt{i}{t}$ to the server.}
    
    \textcolor{violet}{\bf Server} updates the parameter vector  $\model{}{t+1} = \model{}{t} - \learningrate{t} \text{TM}^{(f)}\left(\mmt{1}{t}, \ldots, \, \mmt{n}{t} \right)$ \;
    }
    \Output{$\hat{\model{}{}}$ = $\model{}{T}$}

   \caption{DSGD with distributed momentum and trimmed mean aggregation}
    \label{algorithm:dsgd}
\end{algorithm2e}

The algorithm follows the skeleton of DSGD and imparts robustness to the learning procedure by applying a momentum operation at the workers' level and a trimmed mean operation at the server (instead of averaging), as described in Algorithm~\ref{algorithm:dsgd}. Essentially, in each iteration $t \geq 0$, each honest worker $i$ computes a stochastic gradient 
\begin{align}
     g_t^{(i)} := \nabla q\left( \model{}{t}, x^{(i)} \right), \quad \text{where $x^{(i)} \sim \D^{(i)}$} \enspace,\label{eqn:localgradient}
\end{align}
and
returns a {\em Polyak's momentum} of its stochastic gradients, denoted by $\mmt{i}{t}$ and defined as 
\begin{align}
    \mmt{i}{t} = \beta_t \mmt{i}{t-1} + (1 - \beta_t) g_t^{(i)} \enspace, \label{eqn:mmt_i}
\end{align}
where $\beta_t \in [0, \, 1)$ is the {\em momentum coefficient}, and $\mmt{i}{-1} = 0$ by convention. 
The server updates its current parameter vector $\model{}{t}$ by aggregating the workers' momentums using {\em coordinate-wise trimmed mean} (TM), defined below. Hereafter, for any $z \in \R^d$ and $k \in [d]$, we denote by $[z]_k$ the $k$-th coordinate of $z$. 
Then, given $n$ input vectors $z_1, \dots, \, z_n$ $\in \R^d$, for all $k \in [d]$, we denote by $\tau_k$ the permutation on $[n]$ that sorts the $k$-th coordinates of the input vectors in non-decreasing order, i.e., $[z_{\tau_k(1)}]_k\leq [z_{\tau_k(2)}]_k \leq\ldots \leq [z_{\tau_k(n)}]_k$. Then, the trimmed mean of $z_1, \ldots, \, z_n$, with trimming parameter $f$ is a vector in $\R^d$ whose $k$-th coordinate is defined as follows,
\begin{equation*}
    \left[\text{TM}^{(f)}(z_1, \ldots, z_n)\right]_k \coloneqq \frac{1}{n-2f} \sum_{j \in [f+1,n-f]} [x_{\tau_k(j)}]_k \enspace. \label{eqn:def_CwTM}
\end{equation*}
\subsection{Formal Statement}
\label{sec:convergence}

Theorem~\ref{thm:main} below establishes the convergence of Algorithm~\ref{algorithm:dsgd}, with a Byzantine adversary, 
assuming a scheduled decreasing step sizes and increasing momentum coefficients. Note that the algorithm is oblivious to the identity of faulty workers that may send arbitrary values to the server. We denote by $\expect{\cdot}$ the expectation on the randomness of the algorithm, formally defined in Appendix~\ref{appendix:upperBound}.
\vspace{3mm}

\noindent \fcolorbox{black}{gray!10}{
\parbox{0.97\textwidth}{\centering
\vspace{-5pt}
\begin{theorem}
\label{thm:main}
Suppose assumptions~\ref{asp:lip},~\ref{asp:polyak},~\ref{asp:bnd_var}, and~\ref{asp:heter} hold true. Consider Algorithm~\ref{algorithm:dsgd} with $T \geq 2$ and the following two options for the scheduled step sizes and momentum coefficients.
    \begin{itemize}
        \item \textbf{Option 1:} If $T \leq \frac{54\mu}{L}$, then, $\forall t \in \{0, \dots,T-1 \}$, \quad set $\gamma_t = \frac{1}{18L}$, \quad and \quad $\beta_{t}=0.$
        \item \textbf{Option 2:} If $T > \frac{54\mu}{L}$, then, $\forall t \in \{0, \dots,T-1 \}$, \quad set 
        \begin{center}
            $\gamma_t = \frac{1}{18L + \left[\frac{\mu}{6}(t-t_0+1) \right]^{+}}$ , \quad and \quad $\beta_{t} = 1 - 18L \gamma_{t-1} \enspace.$
        \end{center}
        Where $t_0=\left\lceil \frac{T}{2} \right\rceil$, $\gamma_{-1}=0$ by convention and $[\cdot]^{+}:=\max\{0,\cdot\}$.
    \end{itemize}
Then, the following holds true
\begin{align*}
        \expect{\avgloss\left(\hat{\model{}{}} \right) - \optloss}   \leq \frac{7}{6} Q_0 \cdot e^{-\frac{T}{108K}}+ \left(\lambda+\frac{1}{n-f} \right) \cdot \frac{4374 K \sigma^2}{T\mu}+ \frac{9 \lambda\heter^2}{2\mu} \enspace, 
    \end{align*}
    where $Q_0 := \avgloss(\model{}{0}) - \optloss$, $\lambda = \frac{6 f}{n-2f}\, \left( 1 + \frac{f}{n-2f} \right)$, and $K=\frac{L}{\mu}$. 
\end{theorem}
}
}
\vspace{10pt}

Using Theorem~\ref{thm:main}, we can derive a matching upper bound for Theorem~\ref{thm:lower_bound_strongly_convex} when $K \in \mathcal{O}\left(1\right)$. Specifically, ignoring the constants, we obtain the following corollary.
\begin{corollary}
\label{cor:main}
Suppose $n \geq (2+\nu)f$ for some constant $\nu >0$. Under the conditions stated in Theorem~\ref{thm:main}, Algorithm~\ref{algorithm:dsgd} guarantees that
\begin{align*}
     \expect{\avgloss\left(\hat{\model{}{}} \right) - \optloss} \in \mathcal{O} \left( \frac{f}{n} \cdot \frac{\heter^2}{\mu}  + \varepsilon \right)\enspace,
\end{align*}
with an iteration complexity in 
\begin{align*}
    T \in \mathcal{O}  \left( \textcolor{black!100}{ \frac{1 + f}{n} \cdot \frac{K\sigma^2}{\mu \varepsilon} } + \textcolor{black}{\frac{L}{\mu} \cdot \log\frac{Q_0}{\varepsilon}}\right)\enspace.
\end{align*}
\end{corollary}

\section{Roadmap to Proving Theorem~\ref{thm:main}}
\label{sec:intuition}

We present here the key steps involved in proving Theorem~\ref{thm:main}. Our proof is based on a new Lyapunov function, denoted by $V_t$. We first motivate the design of $V_t$, and define it formally. We then analyze the growth of $V_t$ along the trajectory of Algorithm~\ref{algorithm:dsgd}. Lastly, we show the convergence of the sequence $\left(V_t\right)_{t = 0}^{T-1}$ for the specified diminishing step sizes, thereby proving our result. \\

\paragraph{Analyzing the growth of the loss function.}
We analyze the growth of the loss function $\avgloss (\model{}{t})$ along the trajectory of Algorithm~\ref{algorithm:dsgd}. For any $t \geq 0$ we denote the average momentum of the honest workers as $\AvgMmt{t} \coloneqq \frac{1}{(n-f)} \sum_{i \in \H} \mmt{i}{t} .$
Combining the result of~\cite{allouah2023fixing}  on the robustness of TM with the standard decomposition of the loss function under smoothness assumption (see, e.g.,~\citep{bottou2018optimization}), we get the following bound on the growth of the loss function.

\begin{lemma} \label{lemma:loss_bound}
Suppose Assumption~\ref{asp:lip} holds true. Consider Algorithm~\ref{algorithm:dsgd} with $T\geq 2$, and $\gamma_t \leq \nicefrac{1}{L}$ for all $t \in \{0,\dots,T-1\}$. Let $\lambda$ be as defined in Lemma~\ref{lemma:TM}. Then, for all $t$, the following holds true
\begin{align*}
\expect{\avgloss(\model{}{t+1}) - \avgloss(\model{}{t})}
    &\leq -\frac{\learningrate{t}}{2}\expect{\norm{\nabla  \avgloss(\model{}{t})}^2} +{\learningrate{t}}\frac{\lambda}{n-f}\sum_{i\in\H}\expect{{\norm{\mmt{i}{t}-\AvgMmt{t}}^2}}\\
 &+ {\learningrate{t}}\expect{\norm{ \nabla  \avgloss(\model{}{t})- \AvgMmt{t}}^2} \enspace,
\end{align*}
where $\lambda = \frac{6 f}{n-2f}\, \left( 1 + \frac{f}{n-2f} \right).$
\end{lemma}
From Lemma~\ref{lemma:loss_bound}, we obtained a bound on the growth of the loss function $\avgloss$ during the learning procedure. This lemma highlights the importance of two key quantities: (i) the \emph{deviation} of the average momentum, and (ii) the {\em drift} of each worker $i$ from the average momentum.\\

\paragraph{Incorporating the drift and deviation in the Lyapunov function.}
In the remaining, for any $t \geq 0$, we denote respectively the deviation and the drift of each worker $i$ as 
\begin{align}
    \dev{t} \coloneqq \AvgMmt{t} - \nabla\avgloss (\model{}{t}) \quad \text{ and }\quad \Delta \mmt{i}{t} := \mmt{i}{t} - \AvgMmt{t}, \forall i\in \H \enspace.
\end{align}

Due to the time-varying step size and momentum coefficient in Algorithm~\ref{algorithm:dsgd}, it is difficult to derive a uniform bound (i.e., a bound that holds true for any $t \geq 0$) on the second and third terms in the right hand side of Lemma~\ref{lemma:loss_bound}. Accordingly, we cannot simply and directly analyze the variation of $\expect{ \avgloss\left( \model{}{t} \right) - \optloss}$ with $t$. Instead, we have to incorporate the drift and the deviation in the analysis. Specifically, we define the following Lyapunov function for our problem.
\begin{align}
    \lyap{t} \coloneqq \expect{ \avgloss\left( \model{}{t} \right) - \optloss + \rho \norm{\delta_t}^2+ \rho \frac{\lambda}{n-f}\sum_{i\in\H}\norm{\Delta \mmt{i}{t}}^2}\enspace,\label{eq:lyap_def}
\end{align}
where $\rho = \frac{1}{12L}$.
Then, by definition of~$\lyap{t}$, we have $\expect{\avgloss(\model{}{T}) - \optloss}  \leq \lyap{T}$. Hence an upper bound on $V_T$ gives us an upper bound on $\expect{\avgloss(\model{}{T}) - \optloss}$. With this Lyapunov function at hand, we can construct the proof by following three critical steps: (i) determining a recursive bound on the Lyapunov function $V_t$, (ii) choosing a desirable sequence $(\learningrate{0}, \dots, \learningrate{T-1})$ to obtain tight convergence rate, and (iii) combining (i) and (ii) to derive the final bound on $\expect{\avgloss(\model{}{T}) - \optloss}$.\\

\paragraph{Recursive bound on $V_t$.}
We first  derive a recursive bound for each of the terms in the $V_t$. In doing so, we start by showing in Lemma~\ref{lemma:drift} that the average drift over the honest workers' momentum is controlled by $\beta_{t}$, the gradient diversity $\heter^2$ and the gradient stochasticity $\sigma^2$. 

\begin{lemma}
\label{lemma:drift}
    Suppose assumptions~\ref{asp:bnd_var}, and~\ref{asp:heter} hold true, and consider Algorithm~\ref{algorithm:dsgd} with $T\geq 2$. Then, for any $t \in \{0,\dots, T-1\}$, the following holds true
    \begin{align*}
        &\frac{1}{n-f}\sum_{i\in\H}\expect{\norm{\Delta \mmt{i}{t}}^2}  \leq \beta_{t} \frac{1}{n-f}\sum_{i\in\H}\expect{\norm{\Delta m^{(i)}_{t-1} }^2}
        +(1-\beta_{t})\heter^2+ (1-\beta_{t})^2  \sigma^2 \enspace.
    \end{align*}
\end{lemma}
Next, we study the deviation $\dev{t}$ of the average momentum $\AvgMmt{t}$ from the true gradient $\nabla\avgloss (\model{}{t})$. We obtain in Lemma~\ref{lemma:dev} an upper bound on the growth of the deviation over the steps $t \in \{0,\dots,T-1\}$.
\begin{lemma}
\label{lemma:dev}
Suppose assumptions~\ref{asp:lip},~\ref{asp:bnd_var}, and~\ref{asp:heter} hold true, consider Algorithm~\ref{algorithm:dsgd} with $T\geq 2$, and $\lambda$ as defined in Lemma~\ref{lemma:loss_bound}. Then for any $t \in \{0,\dots,T-1\}$, the following holds true
\begin{align*}
    \expect{\norm{\dev{t+1}}^2} & \leq 
      \beta_{t+1}^2 \left(1+4\learningrate{t} L+3\learningrate{t}^2 L^2 \right )\expect{\norm{\dev{t}}^2} + (1-\beta_{t+1})^2\frac{\sigma^2}{n-f} \\
        &+3\beta_{t+1}^2(\learningrate{t}^2 L^2+\learningrate{t} L)\left(  \frac{\lambda}{n-f}\sum_{i\in\H}{\expect{\norm{\Delta \mmt{i}{t}}^2}} +  \expect{\norm{\nabla\avgloss (\model{}{t})}^2} \right) \enspace.
\end{align*}
\end{lemma}
Finally, combining Lemmas~\ref{lemma:drift} and~\ref{lemma:dev} with Lemma~\ref{lemma:loss_bound}, we can derive a proper recursive bound on $V_t$, as presented in Lemma~\ref{lemma:lyap} below.

\begin{lemma} \label{lemma:lyap}
Suppose assumptions~\ref{asp:lip},~\ref{asp:polyak},~\ref{asp:bnd_var}, and~\ref{asp:heter} hold true. Consider Algorithm~\ref{algorithm:dsgd} with $T\geq 2$ and a set of parameters such that $t \in \{0, \cdots ,T\}$, $\learningrate{t} \leq \frac{1}{18L}$, and $1 - \beta_{t+1} =  18 \learningrate{t} L$. Finally, let $(\lyap{t})_{t \geq 0}$ be as defined in~\eqref{eq:lyap_def} and $\lambda$ as defined in Lemma~\ref{lemma:loss_bound}. Then the following holds true
\begin{align*}
\lyap{t+1}
    &\leq  \left( 1-\frac{\mu \learningrate{t}}{3} \right) \lyap{t} + 27 L \left(\lambda+\frac{1}{n-f} \right) \sigma^2 \learningrate{t}^2 +  \frac{3}{2}  \lambda \heter^2 \learningrate{t} \enspace.
\end{align*}
\end{lemma}

\paragraph{Choice of the step sizes $(\learningrate{0}, \dots, \learningrate{T-1})$.}
To obtain a tight convergence rate (and avoid logarithmic terms), we need to carefully choose the sequence of the step sizes $(\learningrate{0}, \dots, \learningrate{T-1})$ we use, as recently pointed out in \citep{stich19}. Specifically, following the recent advancement on this matter~\citep{khaled20}, we design a generic scheduling technique, described in Lemma~\ref{lem:stepsize_strongly_convex} below.

\begin{lemma}
\label{lem:stepsize_strongly_convex}
        Let $a,b,c,d$ be positive real values with $a < b$, and let $T\geq 2$ be a positive integer. Let $(\learningrate{0}, \dots, \learningrate{T-1})$ and $(r_0, \dots, r_{T})$ be real valued sequences such that for all $t \in \{0,\dots,T-1\}$,
    \begin{align*}
        r_{t+1} \leq (1 - a \learningrate{t}) r_{t} + c \learningrate{t}^2 + d \learningrate{t} \enspace. 
    \end{align*}
    Consider the following two cases: 
    \begin{itemize}
        \item \textbf{Case 1:} $T \leq \nicefrac{b}{a} ~ $ and $\gamma_t = \nicefrac{1}{b}, ~ \forall t \in \{0,\dots,T-1\}$.
        \item \textbf{Case 2:} $T >  \nicefrac{b}{a} ~ $, and $\gamma_t = \frac{1}{b + [\frac{a}{2}(t-t_0+1)]^+ }, ~ \forall t \in \{0,\dots,T-1\}$, where $t_0 = \ceil{\nicefrac{T}{2}}$. \\
    \end{itemize}
    In both Case 1 and Case 2, we have: \quad
   $
        r_T \leq  r_0 \exp \left( - \frac{aT}{2b} \right)  + \frac{18c}{a^2 T}  + \frac{3d}{a} .
  $
\end{lemma}

\noindent\paragraph{Final step for the proof sketch of Theorem~\ref{thm:main}.} Lastly, we apply Lemma~\ref{lem:stepsize_strongly_convex} to the recursion of Lemma~\ref{lemma:lyap}, with $a = \frac{\mu}{3}$, $b = {18L}$, $c = 27 L \left(\lambda+\frac{1}{n-f} \right) \sigma^2$ and $d =  \frac{3}{2}  \lambda  \heter^2$, and obtain that
\begin{align*}
     \lyap{T} \leq \lyap{0} \exp \left( -\frac{\mu T}{108 L} \right)  + \frac{4374L \left(\lambda+\frac{1}{n-f} \right) \sigma^2}{T\mu^2}+ \frac{9 \lambda\heter^2}{2\mu} \enspace.
\end{align*}
As $\expect{\avgloss(\model{}{T}) - \optloss}  \leq \lyap{T}$, we conclude the proof by showing that $V_0 \leq \frac76 \left(\avgloss(\model{}{0}) - \optloss  \right)$.

\section{Partially-Poisonous Local Data}
\label{sec:partially_poisonous}
A standard assumption in robust distributed ML literature that we have also made so far is that each worker is either {\em entirely corrupted} or honest. If a worker is honest then it is assumed that all of its data points are sampled correctly and that it always follows the prescribed algorithm. However, in practice, we might have some corrupted data points among the data points available to all the workers. In particular, instead of considering a fraction of corrupted workers, we may assume a fraction of the data points available to all workers are poisonous (or incorrectly sampled). To address a general data poisoning setting, in this section, we consider both worker-level and global-level data corruptions. Specifically, we assume that the datasets of up to $f$ out of $n$ workers are fully corruptible and that the datasets of remaining $n-f$ workers is partially corruptible. To characterize the impact of these two types of corruptions, we focus on empirical loss minimization where each worker $i$ has a dataset $\SSS^{(i)}$ of $m$ data points.\footnote{A solution to the empirical loss minimization problem is a $\mathcal{O}\left(\frac{\sigma^2}{m}\right)$ approximate solution to the statistical loss.} We assume that $b$ out of $m$ data points of each worker can be arbitrarily corrupted. We let $\D^{(i)}$ denote 
the uniform distribution over the remaining $m - b$ incorruptible data points. By \eqref{eq:local_loss}, we have
\begin{align}
    {\loss}^{(i)}(\model{}{}) \coloneqq \condexpect{x \sim \D^{(i)}}{q(\model{}{}, x)} = \frac{1}{m-b} \sum_{x \in \SSS_h^{(i)}}q(\model{}{}, x) , \label{eqn:erm_loss}
\end{align}
where $ \SSS_h^{(i)} \coloneqq \support{\D^{(i)}}$ is the set of honest data points of worker $i$.
This general data poisoning model encompasses various scenarios. For instance, setting $n = 1$ and $f = 0$ corresponds to the centralized poisoning problem, where a portion of a large dataset is corrupted. Furthermore,  $b = 0$ corresponds to the case where some of the workers are always correct which is the scenario often studied in the Byzantine ML literature that we considered in the previous sections. In the rest of this section, we prove matching upper and lower bounds on the learning error in the above setting.

\begin{remark}
    For the simplicity of presentation, we only consider the data poisoning threat. However, our upper bound holds even for the stronger Byzantine failure threat model, where recall that when a worker is corrupted, it can send an arbitrary vector for its gradient to the server. 
\end{remark}

\subsection{Lower Bound}
\begin{theorem}
\label{thm:lower_gd}
    Consider the average empirical loss~\eqref{eqn:global_loss} with individual loss functions as defined in~\eqref{eqn:erm_loss}. Suppose assumptions~\ref{asp:lip},~\ref{asp:polyak},~\ref{asp:bnd_var}, and~\ref{asp:heter}.
    For any algorithm  $\Pi$ outputting a model $\hat{\theta{}}_{\Pi}$, we have
    $$\avgloss \left( \hat{\theta{}}_{\Pi} \right) - Q^* \in \Omega\left(\frac{f}{n}\cdot\frac{\zeta^2}{\mu} + \frac{b}{m} \cdot\frac{\sigma^2}{\mu}\right).$$
\end{theorem}

\begin{proof}
We prove the theorem for the scalar domain, i.e., $d = 1$, $\X \in \R$, and a quadratic loss, i.e., $q(\model{}{}, x) = \frac{\mu}{4} \left( \model{}{} - x \right)^2$. 
The proof for the first term, i.e., $\Omega\left(\frac{f}{n}\cdot\frac{\zeta^2}{\mu}\right)$, follows from the second case in the proof of Theorem~\ref{thm:lower_bound_strongly_convex}. The second term, i.e., $\Omega\left( \frac{b}{m} \cdot\frac{\sigma^2}{\mu} \right)$ term, also follows from the arguments made in the proof of Theorem~\ref{thm:lower_bound_strongly_convex}. We provide key differences below.

Suppose that $f = 0$, i.e., there is no worker with full-poisonous data in the system. Also, suppose that all the workers have identical local datasets, i.e., $\zeta = 0$. Since, having multiple copies of the same dataset does not provide any additional information, the problem reduces to the case with a single worker possessing a dataset denoted as $\SSS^{(1)}$  such that the honest data points $\SSS_h^{(1)}$ satisfy Assumption~\ref{asp:bnd_var}. Consider a quadratic loss function $q(\model{}{}, x) = \frac{\mu}{4} \left( \model{}{} - x \right)^2$ with gradient $\nabla q(\model{}{}, x) = \frac{\mu}{2} \left( \model{}{} - x\right)$. This loss satisfies  assumptions~\ref{asp:lip} and~\ref{asp:polyak}. For any $j \in [m]$, let $x^{(1,j)}$ be the $j$-th data point in $\SSS^{(1)}$.
Now suppose that $x^{(1,j)} = 0$ for $1 \leq j \leq m-b$, and $x^{(1,j)} = \frac{2\sigma}{\mu}\sqrt{\frac{m-b}{b}}$ for $m-b+1\leq j \leq m$. Consider the following two cases:

\begin{itemize}
    \item[] {\bf Case 1:} $\SSS_h^{(1)} \coloneqq \left\{x^{(1,j)}: 1\leq j \leq m - b \right\}$.
    \item[] {\bf Case 2:} $\SSS_h^{(1)} \coloneqq \left\{x^{(1,j)}: b + 1\leq j \leq m \right\}$.
\end{itemize}

In case 1, we have
\begin{align*}
    \condexpect{x \sim \dist{1}}{\left( \nabla q \left(\model{}{},x\right)-\nabla \localloss{1}(\model{}{}) \right)^2} =\frac{1}{m - b} \sum_{x \in \SSS_h^{(1)} } {\left( \nabla q \left(\model{}{},x\right)-\nabla \localloss{1}(\model{}{}) \right)^2} = 0 \leq \sigma^2  \enspace.   
\end{align*}
In case 2, using the same technique as in Execution 2 of the second case in the proof of Theorem~\ref{thm:lower_bound_strongly_convex}, we have
\begin{align*}
    \condexpect{x \sim \dist{1}}{\left(\nabla q \left(\model{}{},x\right)-\nabla \localloss{1}(\model{}{}) \right)^2} =\frac{1}{m - b} \sum_{x \in \SSS_h^{(1)} } {\left( \nabla q \left(\model{}{},x\right)-\nabla \localloss{1}(\model{}{}) \right)^2} = \sigma^2  \enspace.   
\end{align*}
Therefore, in both cases, Assumption~\ref{asp:bnd_var} is satisfied. 

Now, suppose that algorithm $\Pi$ provides an $\varepsilon$-approximation guarantee on the learning error. Specifically, in both cases, we have
\begin{align*}
    \avgloss \left( \hat{\theta{}}_{\Pi} \right) - Q^* \leq \varepsilon.
\end{align*}
This implies that (refer the first case in the proof of Theorem~\ref{thm:lower_bound_strongly_convex}),
\begin{align*}
    \frac{\mu}{4} \, \left( \hat{\theta{}}_{\Pi} \right)^2 \leq \varepsilon \quad \text{and} \quad \frac{\mu}{4} \, \left(\hat{\model{}{}}_{\Pi} - \frac{2\sigma}{\mu}\sqrt{\frac{b}{m-b}}  \right)^2 \leq \varepsilon . 
\end{align*}
Thus, applying Jensen's inequality, we obtain that
\begin{align*}
    \varepsilon \geq \frac{\mu}{16} \left(\frac{2\sigma}{\mu}\sqrt{\frac{b}{m-b}}\right)^2. 
\end{align*}
The above implies that $\varepsilon \in \Omega\left( \frac{b}{m}\cdot\frac{\sigma^2}{\mu}\right)$.
This concludes the proof.
\end{proof}

\subsection{Upper Bound}
In this section, we establish an upper bound that matches the lower bound presented in Theorem~\ref{thm:lower_gd} by considering Algorithm~\ref{algorithm:d-gd}. Notably, Algorithm~\ref{algorithm:d-gd} exhibits three key distinctions when compared to Algorithm~\ref{algorithm:dsgd}.
Firstly, Algorithm~\ref{algorithm:d-gd} operates deterministically; at each iteration, every worker computes the gradient over its entire dataset, in contrast to the stochastic nature of Algorithm~\ref{algorithm:dsgd}.
Secondly, in addition to the global aggregation functions performed by the server, each worker in Algorithm~\ref{algorithm:d-gd} incorporates a locally applied trimmed mean aggregation function. This function serves to filter out outliers, ensuring the robustness of the local updates.
Finally, Algorithm~\ref{algorithm:d-gd} does not require local momentum (owing to its deterministic nature), and the model is updated using robustified gradient vectors. The following theorem shows the convergence of Algorithm~\ref{algorithm:d-gd}. The proof can be found in Appendix~\ref{app:proofupperboundgd}.

\begin{algorithm2e}[!htb]
    
    \SetKwInOut{Input}{Input}
    \SetKwInOut{Output}{Output}
    \SetAlgoLined
    \Input{$T \geq 2$, step size $\gamma > 0$.}
    \textcolor{violet}{\bf Server} chooses arbitrarily $\theta_0 \in \R^d$.
    
    \For{$t=0$ to $T-1$}{
    \textcolor{violet}{\bf Server} broadcasts $\model{}{t}$ to all workers\;
    
    \For{each \textcolor{teal}{\bf honest worker} $i$ (in parallel)}{
    
        for each data point $x \in  \mathcal{S}^{(i)}$, compute $\nabla q \left(\model{t}{},x\right)$, computes
    $$G_t^{(i)}:= \text{TM}^{(b)}\left(\nabla q \left(\model{t}{},x\right), \forall x \in \mathcal{S}^{(i)}\right),$$
    and sends $G_t^{(i)}$ to the server.

    }
    
    \textcolor{darkgray}{\% A corrupted worker $i$ 
    may send an arbitrary vector to the server. 
    }
    
    \textcolor{violet}{\bf Server}
    updates the parameter vector  $\model{}{t+1} = \model{}{t} - \gamma \text{TM}^{(f)}\left(G_t^{(1)}, \ldots, \, G_t^{(n)} \right)$ \;
    }
    \Output{$\hat{\model{}{}}$ = $\model{}{T}$}

   \caption{DGD with local and global trimmed mean aggregations}
    \label{algorithm:d-gd}
\end{algorithm2e}

\begin{theorem}
\label{thm:upperboundgd}
    Consider the average empirical loss~\eqref{eqn:global_loss} with individual loss functions as defined in~\eqref{eqn:erm_loss}. Suppose assumptions~\ref{asp:lip},~\ref{asp:polyak},~\ref{asp:bnd_var}, and~\ref{asp:heter}. Consider Algorithm~\ref{algorithm:d-gd} with $\gamma = 1/L$. Then, we have 
    \begin{align*}
        \avgloss(\model{}{T}) - Q^* &\leq \exp\left(-\frac{\mu}{L}T\right) \left( \avgloss(\model{}{0}) - Q^* \right) + \frac{1}{\mu}(  \lambda' \sigma^2  + 3 \lambda\lambda' \sigma^2 + 3 \lambda\zeta^2) ,
    \end{align*}
    where $\lambda = \frac{6 f}{n-2f}\, \left( 1 + \frac{f}{n-2f} \right)$ and $\lambda' = \frac{6 b}{m-2b}\, \left( 1 + \frac{b}{m-2b} \right)$.
\end{theorem}

Note that $\lambda'  \in \mathcal{O}\left(\frac{b}{m}\right)$ and $\lambda \in \mathcal{O}\left(\frac{f}{n}\right)$. Hence, we obtain the following corollary of Theorem~\ref{thm:upperboundgd}.
\begin{corollary}
Under the same conditions as in Theorem~\ref{thm:upperboundgd}, Algorithm~\ref{algorithm:d-gd} outputs $\model{}{T}$ such that 
$$\avgloss(\model{}{T}) - Q^* \in \mathcal{O}\left(\frac{f}{n}\cdot\frac{\zeta^2}{\mu} + \frac{b}{m} \cdot\frac{\sigma^2}{\mu}+ \varepsilon \right),$$
as long as $T \in \mathcal{O} \left( \frac{L}{\mu} \cdot \log\frac{Q_0}
{\varepsilon} \right)$.
\end{corollary}

\section{Concluding Remarks \& Open Problems}
\label{sec:remarks}
We have shown that the Byzantine failure threat model is not an overkill for addressing the more practical threat model of data poisoning.
Specifically, we have shown that state-of-the-art solutions to the {\em Byzantine ML} problem, such as the ones proposed in~\cite{farhadkhani2022byzantine, karimireddy2022byzantine, allouah2023fixing, gorbunov2023variance}, provide optimal protection against data poisoning attacks. Although our result applies to ML problems that are solvable by optimizing over Polyak-\L{}ojasiewicz (PL) loss functions, we believe that our deductions hold true even for a larger set of functions that do not necessarily satisfy the PL inequality. This constitutes an interesting future research direction. Furthermore, we have also shown that Byzantine robustness schemes yield tight solutions in both partial-poisonous and full-poisonous local data settings.

Note that we have only considered {\em untargeted} attacks in both the Byzantine failure and the data poisoning threat models. An interesting future direction would be to consider {\em targeted} attacks, wherein corrupted workers do not necessarily attempt to maximize the learning error, but rather act strategically to manipulate the learning into converging to a target region in the model space that performs poorly on specific types of inputs (i.e., has high generalization errors), e.g., see~\citep{DaiCL19,wang2020,ZhaoMZ0CJ20,TruongJHAPJNT20,SeveriMCO21}. While a recent work has attempted to compare Byzantine failure and data poisoning in the context of targeted attacks~\citep{equivalenceposiong}, the findings only applicable to conventional ML methods that do not incorporate any robustness properties. Our proof techniques could be used to obtain a principled comparison between the two threat models in the targeted attacks scenario.

\bibliographystyle{abbrv}
\bibliography{arxiv}

\newpage
\appendix

\begin{center}
    \LARGE \bf {Appendix}
\end{center}

\section{Proof of Theorem~\ref{thm:lower_bound_strongly_convex}}
\label{app:lower_bnd}

    \begin{remark}
        Note that, to prove Theorem~\ref{thm:lower_bound_strongly_convex}, we focus on the special case where $d = 1$. As the lower bound is established using the squared Euclidean norm, the proof applies
directly to $d > 1$ since the instances used in the proof are still valid in a $1$-dimensional subspace.
        Moreover, as we later prove in Corollary~\ref{cor:main}, this lower bound is tight as it is matched by Algorithm~\ref{algorithm:dsgd} for an arbitrary $d$. Note, however, that despite the explicit absence of the dimension $d$ in the asymptotic error and the convergence rate, the impact of dimension $d$ is implicit through $\sigma^2$, i.e., the bound stated in Assumption~\ref{asp:bnd_var} on the {\em covariance trace} of the local stochastic noise. Indeed, when the variance of noise in each coordinate of the stochastic gradients might be as large as some real value $\varsigma^2$, we have $\sigma^2 = d \cdot \varsigma^2$ .
    \end{remark}

    To prove Theorem~\ref{thm:lower_bound_strongly_convex}, we need to show that for any $T > 0$, and any algorithm $\Pi$, we must have\footnote{Here we ignore the absolute constant in the exponent as it corresponds to a constant multiplied by the logarithmic term in Theorem~\ref{thm:lower_bound_strongly_convex}.} $$\condexpect{\Pi}{\avgloss \left( \hat{\model{}{}} \right) - Q^*} \in \Omega \left( \frac{f+1}{n} \cdot \frac{\sigma^2}{\mu T} + \frac{f}{n} \cdot 
        \frac{\heter^2}{\mu} + e^{-\frac{T}{K}} \right).$$

     We assume that the output $\hat{\model{}{}}$ of algorithm $\Pi$ satisfies the condition: $\condexpect{\Pi}{\avgloss \left( \hat{\model{}{}} \right) - Q^*} \leq A$ for $A > 0$. To obtain a lower bound on $A$, we consider a setting where $d = 1$, $\X \subseteq \R$ and the loss function
    $q(\model{}{}, \, x) = \frac{\mu}{4}(\model{}{}-x)^2$ where $0 < \mu < \infty$. We consider two separate cases, each with different instances of data distributions subject to assumptions~\ref{asp:lip},~\ref{asp:polyak},~\ref{asp:bnd_var}, and~\ref{asp:heter}.

    In the first case, we consider heterogeneous distributions for honest workers, i.e., $\heter \geq 0$. In this particular case, we adapt the proof of Theorem III in~\cite{karimireddy2022byzantine} to show that 
    \begin{align}
        A  \in \Omega\left(\frac{f}{n} \cdot 
        \frac{\heter^2}{\mu}\right). \label{eqn:ep_2}
    \end{align}
    In the second case, we assume $\D^{(i)} = \D$ for all $i \in \H$, i.e., $\heter = 0$ in Assumption~\ref{asp:heter}. In this particular case, we develop upon the indistinguishability of valid distributions in the general contamination model (shown in Proposition 1.7 of~\cite{dia22}) to show that
    \begin{align}
        A \in \Omega \left( \frac{f+1}{n} \cdot \frac{\sigma^2}{\mu T}  \right).  \label{eqn:ep_1}
    \end{align}
    As $\hat{\model{}{}}$ should satisfy the bound in both cases, {\bf the proof concludes upon combining~\eqref{eqn:ep_1} and~\eqref{eqn:ep_2}.}
     and the recently discovered $\Omega( e^{-\frac{T}{K}} )$ lower bound for first-order deterministic algorithms~\cite{yue2023lower} in the vanilla (non-Byzantine) setting.\footnote{Follows from the fact that $\max\{a, b, c\} \geq \frac{1}{3}(a + b + c).$} \\

        \noindent {\bf First Case.} In this case, we obtain a bound on the error when honest workers may have non-identical data distributions. Our derivation follows from the proof of Theorem III in \cite{karimireddy2022byzantine}. We partition the set of workers into $S = \{1, \ldots, \, n-f\}$ and $\hat{S} = \{n-f+1, \ldots, \, n\}$. We consider the following Dirac distributions of data.
    \begin{align*}
        \text{Distribution} \quad \D^{(i)} : \begin{cases}
          x = 0  ~ \text{ with probability } ~ 1 ~ ; & i \in S\\
          x = \frac{2 \heter}{\mu}\sqrt{\frac{n-f}{f}} ~ \text{ with probability } ~ 1 ~ ; & i \in \hat{S}
      \end{cases} \quad.
    \end{align*}
    Next, we consider two valid executions of $\Pi$ with different identities for honest workers. In Execution 1, $\H = S$ and in Execution 2, $\H = \{1, \ldots, \, n-2f\} \cup \hat{S}$. It is easy to verify (using similar steps as in the first case) that assumptions~\ref{asp:lip},~\ref{asp:polyak} and~\ref{asp:bnd_var} are satisfied in either executions. We show below that Assumption~\ref{asp:heter} is also satisfied in the two executions. Hence, validating both the executions. Recall that we assume $f < \frac{n}{2}$.\\

    Note that $\localloss{i}(\model{}{}) \coloneqq \frac{\mu}{4} {\model{}{}}^2$ for all $i \in S$, and $\localloss{i}(\model{}{}) \coloneqq \frac{\mu}{4} \left(\model{}{} - \frac{2 \heter}{\mu}\sqrt{\frac{n-f}{f}} \right)^2$ for all $i \in \hat{S}$. In {\bf Execution 1}, as $\H = S$, it is easy to see that
    \begin{align}
        \frac{1}{\card{\H}}\sum_{i \in \H} {\norm{\nabla\localloss{i}(\model{}{})-\nabla\avgloss(\model{}{})}^2} = \frac{1}{\card{S}}\sum_{i \in S} \left( \frac{\mu}{2} \model{}{} - \frac{1}{\card{S}} \sum_{j \in S}\frac{\mu}{2} \model{}{}\right)^2= 0 \leq \heter^2. \label{eqn:heter-1}
    \end{align}
    In {\bf Execution 2}, as $\H = \{1, \ldots, \, n-2f\} \cup \hat{S}$, we have 
    \begin{align*}
        \avgloss(\model{}{}) = \frac{\mu (n-2f)}{4 (n-f)} {\model{}{}}^2 + \frac{\mu f}{4 (n-f)} \left( \model{}{} - \frac{2 \heter}{\mu \sqrt{\frac{n-f}{f}}}\right)^2 = \frac{\mu}{4} \left( \model{}{} - \frac{2 \heter}{\mu} \sqrt{\frac{f}{n-f}}\right)^2 + \frac{n-2f}{n-f} \cdot \frac{\heter^2}{\mu}.
    \end{align*}
    Therefore,
    \begin{align*}
        \nabla\avgloss(\model{}{}) = 
        \frac{\mu}{2} \left( \model{}{} - \frac{2 \heter}{\mu} \sqrt{\frac{f}{n-f}} \right) .
    \end{align*}
    Thus,
    \begin{align*}
        & \frac{1}{\card{\H}}\sum_{i \in \H} {\norm{\nabla\localloss{i}(\model{}{})-\nabla\avgloss(\model{}{})}^2} = \frac{1}{n-f} \sum_{i = 1}^{n-2f} \left( \frac{\mu}{2} \model{}{} - \frac{\mu}{2} \left( \model{}{} - \frac{2 \heter}{\mu} \sqrt{\frac{f}{n-f}} \right) \right)^2 \nonumber\\ 
        & + \frac{1}{n-f} \sum_{i \in \hat{S}} \left( \frac{\mu}{2} \left(\model{}{} - \frac{2 \heter}{\mu}\sqrt{\frac{n-f}{f}} \right)- \frac{\mu}{2} \left( \model{}{} - \frac{2 \heter}{\mu} \sqrt{\frac{f}{n-f}} \right) \right)^2 
    \end{align*}
    Upon simplifying the RHS above we obtain that
    \begin{align}
       \frac{1}{\card{\H}}\sum_{i \in \H} {\norm{\nabla\localloss{i}(\model{}{})-\nabla\avgloss(\model{}{})}^2} = \frac{f(n-2f)}{(n-f)^2} \heter^2 + \frac{(n-2f)^2}{(n-f)^2} \heter^2 = \frac{n - 2f}{n - f} \, \heter^2 \leq \heter^2. \label{eqn:heter-2}
    \end{align}
    Thus, due to~\eqref{eqn:heter-1} and~\eqref{eqn:heter-2}, Assumption~\ref{asp:heter} is also satisfied in both executions. \\

    Recall that in each execution of algorithm $\Pi$ the output $\hat{\model{}{}}$ satisfies the condition: $\condexpect{\Pi}{\avgloss \left( \hat{\model{}{}} \right) - Q^*} \leq A$. Thus, from Execution 1, as $Q^* = 0$ and $\avgloss \left( \model{}{} \right) \coloneqq \frac{\mu}{4} {\model{}{}}^2$, we have 
    \begin{align}
        \frac{\mu}{4} \, \condexpect{\Pi}{\hat{\model{}{}}^2} \leq A. \label{eqn:exec-i}
    \end{align}
    Similarly, from Execution 2, as $Q^* = \frac{n-2f}{n-f} \cdot \frac{\heter^2}{\mu}$ and $\avgloss \left( \model{}{} \right) \coloneqq \frac{\mu}{4} \left( \model{}{} - \frac{2 \heter}{\mu} \sqrt{\frac{f}{n-f}}\right)^2 + \frac{n-2f}{n-f} \cdot \frac{\heter^2}{\mu}$, we obtain that
    \begin{align}
        \frac{\mu}{4} \, \condexpect{\Pi}{\left( \hat{\model{}{}} - \frac{2 \heter}{\mu} \sqrt{\frac{f}{n-f}}\right)^2}  \leq A\label{eqn:exec-ii}
    \end{align}
    From Jensen's inequality, as $a^2 \leq (a - b + b)^2 \leq 2(a-b)^2 + 2b^2$, we have
    \begin{align}
        \left(\frac{2\heter}{\mu} \sqrt{\frac{f}{n-f}} \right)^2 \leq 2 \left(\frac{2 \heter}{\mu} \sqrt{\frac{f}{n-f}} - \hat{\model{}{}}\right)^2 + 2 \hat{\model{}{}}^2.
        \label{eq:inter998}
    \end{align}
    Upon substituting from~\eqref{eqn:exec-i} and~\eqref{eqn:exec-ii} in the above, we obtain that
    \begin{align*}
        \left(\frac{2\heter}{\mu} \sqrt{\frac{f}{n-f}} \right)^2 \leq \frac{16}{\mu} \, A.
    \end{align*}
    From above, we obtain that $A ~ \geq ~ \frac{f}{n-f} \cdot \frac{\heter^2}{4 \mu}~ \geq \frac{f}{n} \cdot \frac{\heter^2}{4 \mu}$, which implies~\eqref{eqn:ep_2}, i.e., 
    \begin{align*}
        A \in \Omega\left(\frac{f}{n} \cdot 
        \frac{\heter^2}{\mu}\right)
    \end{align*}
    This completes the proof of Theorem~\ref{thm:lower_bound_strongly_convex}. \\
    
    \noindent {\bf Second Case.} Let $\D^{(i)} = \D$ for all $i \in \H$, where distribution $\D$ satisfies the following:
    \begin{align*}
        \condexpect{x \sim \D}{x} < \infty, ~ \text{ and } ~ \condexpect{x \sim \D}{\left( x - \condexpect{x \sim \D}{x} \right)^2} \leq \frac{4 \sigma^2}{\mu^2}.
    \end{align*}
    By definition of $\loss^{(i)} (\model{}{})$, we obtain that for all $i$,
    \begin{align}
        \loss^{(i)} (\model{}{})
        = \frac{\mu}{4} \, \condexpect{x \sim \D}{\left( \model{}{}-x \right)^2}, ~ \text{ and thus, } ~ \nabla \loss^{(i)} (\model{}{}) = \frac{\mu}{2} \, \left( \theta - \condexpect{x \sim \D}{x} \right). \label{eqn:case1-1}
    \end{align}
    Thus, Assumption~\ref{asp:lip} holds true, i.e., $\nabla \loss^{(i)} (\model{}{})$ is Lipschitz continuous, with $L = \mu$. Assumption~\ref{asp:bnd_var} holds true due to the following:
    \begin{align*}
         \condexpect{x \sim \D}{\left(\nabla \loss^{(i)} (\model{}{})  -  \nabla q(\model{}{}, \, x) \right)^2 } & = \condexpect{x \sim \D}{\left(\frac{\mu}{2} \left( \theta - \condexpect{x \sim \D}{x} \right) - \frac{\mu}{2} \left(\model{}{} - x \right) \right)^2} \\
         & = \frac{\mu^2}{4} \, \condexpect{x \sim \D}{\left(x - \condexpect{x \sim \D}{x} \right)^2} = \sigma^2.
    \end{align*}
    From~\eqref{eqn:case1-1}, we obtain that
    \begin{align}
        \avgloss\left( \model{}{} \right) \coloneqq \frac{1}{\card{\H}} \sum_{i \in \H} \loss^{(i)} (\model{}{}) =  \frac{\mu}{4} \condexpect{x \sim \D}{\left(\model{}{}-x \right)^2}. \label{eqn:case1_avg}
    \end{align}
    Thus, $\loss^{(i)}$ and $\avgloss$ are identical in this case, and Assumption~\ref{asp:heter} holds true trivially for $\zeta = 0$. From above we obtain that $\theta^* \coloneqq  \argmin_{\model{}{} \in \R^d} \, \avgloss (\model{}{})  = \condexpect{x \sim \D}{x}$, and thereby,
    \begin{align}
        Q^* = \avgloss\left( \theta^* \right) = \frac{\mu}{4} \, \condexpect{x \sim \D}{\left(x - \condexpect{x \sim \D}{x}\right)^2} . \label{eq:mean_estimation}
    \end{align} 
    From~\eqref{eqn:case1_avg} and~\eqref{eq:mean_estimation} we obtain that 
    \begin{align}
        \avgloss\left( \model{}{} \right) - Q^* = \frac{\mu}{4} \left( \model{}{} - \condexpect{x \sim \D}{x} \right)^2. \label{eqn:mean_estimation_epsilon}
    \end{align}
    Thus,  
    \[ \left( \nabla \avgloss\left( \model{}{} \right)  \right)^2 = 
    \frac{\mu^2}{4} \left(\model{}{} -  \condexpect{x \sim \D}{x} \right)^2 = \mu \left(\avgloss\left( \model{}{} \right) - Q^* \right).\]
    Therefore, Assumption~\ref{asp:polyak} also holds true. \\

    We show that the accuracy of Algorithm $\Pi$ reduces to that of an algorithm for estimating the mean of $\D$ by processing $n$ batches of $T$ points; $n-f$ batches sampled from $\D^T$ but the remainder $f$ batches may contain arbitrary points. From~\eqref{eqn:mean_estimation_epsilon} we obtain that 
    \begin{align*}
        \condexpect{\Pi}{\avgloss\left( \hat{\model{}{}} \right) - Q^*} = \frac{\mu}{4} \condexpect{\Pi}{\left( \hat{\model{}{}} - \condexpect{x \sim \D}{x} \right)^2}.
    \end{align*}
    Recall that we assume that $\condexpect{\Pi}{\avgloss\left( \hat{\model{}{}} \right) - Q^*} \leq A$. Thus, from above we have 
    \begin{align}
        A \geq \frac{\mu}{4} \, \condexpect{\Pi}{\left( \hat{\model{}{}} - \condexpect{x \sim \D}{x} \right)^2}. \label{eqn:error_Pi}
    \end{align}
    The above implies that Algorithm $\Pi$ can estimate the mean of distribution $\D$ within a sqaured-error of $4A/\mu$. Recall that in algorithm $\Pi$, each honest worker $i \in \H$ computes $T$ local stochastic gradients $\{\nabla q (\model{}{t}, \, x^{(i)}_t) ~ ; ~ t = 1, \ldots, \, T\}$ where each element in the set of observations $X^{(i)} \coloneqq \{ x^{(i)}_t ~ ; ~ t = 1, \ldots, \, T \}$ is i.i.d.~from the distribution $\D^{(i)} = \D$. 
    Recall that $\nabla q(\model{}{}, \, x) = \frac{\mu}{2}(\theta - x)$. Therefore, given the value of $\mu$, the set of parameter vectors $\{\model{}{t} ~ ; ~ t \in [T]\}$, we can recover the collection of random observations $\left\{ X^{(i)} ~ ; ~ i \in \H \right\}$ where $ X^{(i)} \sim \D^T$. Hence, it is obvious that the squared error for the mean estimation of $\D$ obtained upon executing $\Pi$ cannot be smaller than that of an {\em optimal} (possibly randomized) robust mean estimator $\Pi_{mean}$ that takes in as inputs $n$ sets of random values $X^{(1)}, \ldots, \, X^{(n)}$ such that $X^{(i)} \sim \D^T$ for all $i \in \H$ and $X_i$ for $i \in [n] \setminus \H$ may be an arbitrarily tuple of $T$ points. 
    Specifically, let $\hat{x} = \Pi_{mean} \left(X^{(1)}, \ldots, \, X^{(n)} \right)$, then
    \begin{align}
        \frac{4 A}{\mu} \geq \condexpect{\Pi_{mean}}{\left( \hat{x} - \condexpect{x \sim \D}{x} \right)^2}. \label{eqn:pi_pi'}
    \end{align}
    We obtain in the following a lower bound on the squared-error $\left( \hat{x} - \condexpect{x \sim \D}{x} \right)^2$ reasoning by indistinguishability of correct distributions under Huber's contamination model. Suppose there exists a distribution $D'$ such that the variance of $\D'$ is also upper bounded by $\frac{4 \sigma^2}{\mu^2}$ (same as that for $\D$) and $\text{TV}(\D^T,\D'^T) \leq \frac{2f}{n}$. Then, by virtue of Proposition 1.7 in \cite{dia22}, no algorithm can reliably distinguish whether the sets of observations $\left\{ X^{(i)} ~ ; ~ i \in \H \right\}$ were generated from $\D^T$ or $\D'^T$. Therefore,
    \begin{align}
        \condexpect{\Pi_{mean}}{\left(\hat{x} - \condexpect{x \sim \D}{x} \right)^2} \geq \frac{1}{4} \, \left( \condexpect{x \sim \D}{x} - \condexpect{x \sim \D'}{x} \right)^2. \label{eqn:mean_d-d'}
    \end{align}
    We construct the following valid distributions $\D$ and $\D'$ to obtain a lower bound for the RHS in~\eqref{eqn:mean_d-d'}.
    \begin{align*}
        \text{Distribution } ~ \D: & \quad x = 0 \text{ \quad with probability } ~ 1.\\
        \text{Distribution } ~ \D': & \quad x =
        \begin{cases}
          \frac{2\sigma}{\mu} \sqrt{\frac{Tn}{2f}}  &\text{ \quad with probability } ~ \frac{2f}{nT} \\
          0 &\text{ \quad with probability } ~ 1 - \frac{2f}{nT} 
      \end{cases} \quad.
    \end{align*}
    {\bf Validity of $\D$ and $\D'$.} Note that $\condexpect{x \sim \D}{x} = 0$, and variance $\condexpect{x \sim \D}{\left( x - \condexpect{x \sim \D}{x} \right)^2} = 0 \leq \frac{4 \sigma^2}{\mu^2}$.
    Similarly, $\condexpect{x \sim \D'}{x} = \frac{2\sigma}{\mu} \sqrt{\frac{2f}{nT}}$ and variance $\condexpect{x \sim \D'}{\left( x - \condexpect{x \sim \D}{x} \right)^2} = \frac{4 \sigma^2}{\mu^2} (1-\frac{2f}{nT}) \leq \frac{4 \sigma^2}{\mu^2}$. Let $0^T$ denote a $T$-tuple with all elements equal to $0$. If $X \sim \D'^T$ then
    $$ \text{Pr} (X = 0^T) = \left(1 - \frac{2f}{nT}\right)^T .$$
    As $\left(1 - \frac{2f}{nT}\right)^T \geq 1 - \frac{2f}{n}$, from above we obtain that $\text{TV}(\D^T,\D'^T) = 1 - \left(1 - \frac{2f}{nT}\right)^T \leq \frac{2f}{n}$. Therefore, $\D$ and $\D'$ are indistinguishable.\\

    Substituting the mean values of $\D$ and $\D'$ in~\eqref{eqn:mean_d-d'} we obtain that
    \begin{align*}
        \condexpect{\Pi_{mean}}{\left(\hat{x} - \condexpect{x \sim \D}{x} \right)^2} \geq \frac{1}{4} \, \left( \frac{2\sigma}{\mu} \sqrt{\frac{2f}{nT}} \right)^2 = \frac{2 \sigma^2}{\mu^2} \cdot \frac{f}{nT} ~.
    \end{align*}
    Substituting from above in~\eqref{eqn:pi_pi'} we have
    \begin{align}
        \frac{4A}{\mu} \in \Omega\left(\frac{f}{n} \cdot \frac{\sigma^2}{\mu^2 T}   \right) ~ . \label{eqn:pi_pi'-2}
    \end{align}
    As we have at most $nT$ samples drawn from distribution $\D$, by the {classical lower bound on statistical error rate~(see Section 3.2 of \cite{wu2017lecture})}, we also have
    \begin{align}
         \frac{4A}{\mu} \in \Omega\left( \frac{1}{n} \cdot \frac{\sigma^2}{\mu^2 T} \right). \label{eqn:cramer-rao}
    \end{align}
    Finally, combining~\eqref{eqn:pi_pi'-2} and~\eqref{eqn:cramer-rao} we obtain~\eqref{eqn:ep_1}, i.e., 
    \begin{align*}
        A \in \Omega \left( \frac{f + 1}{n} \cdot \frac{\sigma^2}{\mu T} \right) . 
    \end{align*}

\section{Deferred Proofs for Theorem~\ref{thm:main}}
\label{appendix:upperBound}
Before proving a few simple lemmas that will be used in the subsequent proofs, let us introduce some useful notations.

\noindent{\bf Notation:}
We denote by $\mathcal{P}_t$ the history of nodes from steps $0$ to $t$. Specifically, we define
\[\P_t \coloneqq \left\{\model{}{0}, \ldots, \, \model{}{t}; ~ \mmt{i}{1}, \ldots, \, \mmt{i}{t-1}; i = 1, \ldots, \, n \right\}.\] 
By convention, $\P_0 = \{ \model{}{0} \}$. Furthermore, we denote by $\condexpect{t}{\cdot} := \expect{\cdot ~ \vline ~ \P_t}$ the conditional expectation given the history $\mathcal{P}_t$, and by $\expect{\cdot}$ the total expectation over the randomness of the algorithm; thus, $\expect{\cdot} := \condexpect{0}{ \cdots \condexpect{T}{\cdot}}$.
Also denote by

\begin{align}
\label{eqn:defavgmomentumandRt}
     R_t \coloneqq \text{TM}\left(\mmt{1}{t}, \ldots, \, \mmt{n}{t} \right) \enspace,
\end{align}
the output of trimmed mean operation.
\subsection{Preliminary Lemmas}
Note that by decomposing the update rule computed by the server at step $t$, we can treat Algorithm~\ref{algorithm:dsgd} as DSGD with a momentum term and a bias $\gamma_t \left( R_t -   \, \AvgMmt{t} \right)$. Specifically, we have
\begin{align}
\label{eqn:biasdecomposition}
     \model{}{t+1} = \model{}{t} - \learningrate{t} R_t = \model{}{t} - \learningrate{t}   \, \AvgMmt{t} - \learningrate{t} \left(R_t -   \, \AvgMmt{t} \right) \enspace.
\end{align}
The key to better understand the bias term in~\eqref{eqn:biasdecomposition} is the analysis of TM, that attempts to robustly estimate the average of the honest momentums at every step. Using a recent result in~\citep{allouah2023fixing}, we can actually bound the bias $(R_t - \AvgMmt{t})$ from above by the spread of honest nodes' momentums.
Specifically, we have the following lemma.
\begin{lemma}[Proposition 2 in~\citep{allouah2023fixing}]
\label{lemma:TM}
Let $n>2f$. Consider Algorithm~\ref{algorithm:dsgd} and $R_t$ as defined in~\eqref{eqn:defavgmomentumandRt}. For any $t\geq 0$, we have
\begin{align*}
    \norm{R_t - \AvgMmt{t}}^2 \leq \frac{\lambda}{n-f}\sum_{i\in\H}{\norm{\mmt{i}{t}-\AvgMmt{t}}^2},\quad \text{with} \quad \lambda = \frac{6 f}{n-2f}\, \left( 1 + \frac{f}{n-2f} \right) \enspace.
\end{align*}
\end{lemma}

We also prove two useful lemmas.
\begin{lemma}
\label{lemma:smootnessbound}
Suppose Assumption~\ref{asp:lip}, i.e., $\avgloss$ is Lipschitz smooth with coefficient $L$. We denote $\optloss = \min_{\model{}{} \in \R^d} \avgloss(\model{}{})$. For all $\model{}{} \in \mathbb{R}^d$, we have$$\Vert\nabla \avgloss(\model{}{}) \Vert^2 \leq 2L (\avgloss(\model{}{})- \optloss)\enspace.$$
\end{lemma}

\begin{proof}
As $\avgloss \coloneqq \frac{1}{\card{\H}} \sum_{i \in \H} \localloss{i}$, Assumption~\ref{asp:lip} implies that for all $\model{}{}$ and $\modelp$, $$\norm{\nabla \avgloss(\model{}{}) - \nabla \avgloss(\modelp)} \leq L \norm{\model{}{} - \modelp}.$$ Thus, from Lipschitz inequality, for all $\model{}{},\modelp \in \mathbb{R}^d$~\citep{bottou2018optimization}, 
\begin{align*}
    \avgloss(\modelp) \leq \avgloss(\model{}{}) + \langle \nabla \avgloss(\model{}{}), \modelp-\model{}{}\rangle+\frac{L}{2} \Vert \modelp- \model{}{} \Vert^2\enspace.
\end{align*}
Consider an arbitrary $\model{}{} \in \R^d$, and let $\modelp = \model{}{} - \frac{1}{L} \nabla \avgloss(\model{}{})$. Thus, from above we obtain that
\begin{align*}
    \avgloss\left(\model{}{} - \frac{1}{L} \nabla \avgloss(\model{}{})\right) &\leq \avgloss(\model{}{})   -  \frac{1}{L}  \Vert\nabla \avgloss(\model{}{}) \Vert^2  +  \frac{1}{2L}  \Vert\nabla \avgloss(\model{}{}) \Vert^2 \\
    & = \avgloss(\model{}{}) -  \frac{1}{2L}  \Vert\nabla \avgloss(\model{}{}) \Vert^2\enspace.
\end{align*}
As $\optloss = \min_{\R^d} \avgloss$, we have
\begin{align*}
    \optloss \leq \avgloss\left(\model{}{} - \frac{1}{L} \nabla \avgloss(\model{}{})\right) \leq \avgloss(\model{}{}) -  \frac{1}{2L}  \Vert\nabla \avgloss(\model{}{}) \Vert^2\enspace.
\end{align*}
Rearranging the terms we obtain that
\begin{align*}
     \Vert\nabla  \avgloss(\model{}{}) \Vert^2 \leq 2L ( \avgloss(\model{}{})- Q^*)\enspace.
\end{align*}
\end{proof}

\begin{lemma}
\label{lem:diam_equal}
    Consider an arbitrary non-empty set $S \subseteq \{1, \ldots, \, n\}$. For any set of $\card{S}$ real-valued vectors $\{x^{(i)}\}_{i \in S}$, we obtain that $$\frac{1}{\card{S}} \sum_{i\in S} \norm{ x^{(i)}- \bar{x}}^2 =  \frac{1}{2 \card{S}^2} \sum_{i,j \in S} \norm{ x^{(i)}- x^{(j)}}^2, \quad \text{where} \quad \bar{x} = \frac{1}{\card{S}} \sum_{i \in S} x^{(i)}\enspace.$$
\end{lemma}
\begin{proof}
\begin{align*}
     \frac{1}{\card{S}^2} \sum_{i,j \in S} \norm{ x^{(i)}- x^{(j)}}^2 &= \frac{1}{\card{S}^2} \sum_{i,j \in S} \norm{ (x^{(i)} -\bar{x}) - (x^{(j)} - \bar{x})}^2\\ &= \frac{1}{\card{S}^2}\sum_{i,j \in S}\left[ \norm{x^{(i)} -\bar{x} }^2 + \norm{x^{(j)} -\bar{x} }^2 + 2 \iprod{x^{(i)} -\bar{x}}{x^{(j)} -\bar{x}}\right] \\
     &= \frac{2}{\card{S}} \sum_{i,j \in S} \norm{x^{(i)} -\bar{x} }^2  + \frac{2}{\card{S}^2}  \sum_{i \in S} \iprod{x^{(i)} - \bar{x}}{ \sum_{j \in S} (x^{(j)} -\bar{x})}\enspace.
\end{align*}
As $\sum_{j \in S} (x^{(j)} -\bar{x}) = 0$, from above we obtain that
\begin{align*}
    \frac{1}{\card{S}^2} \sum_{i,j \in S} \norm{ x^{(i)}- x^{(j)}}^2 = \frac{2}{\card{S}} \sum_{i,j \in S} \norm{x^{(i)} -\bar{x} }^2\enspace.
\end{align*}
\end{proof}

\subsection{Proof of the lemmas provided in the main paper}

\begin{replemma}{lemma:loss_bound}
Suppose Assumption~\ref{asp:lip} holds true. Consider Algorithm~\ref{algorithm:dsgd} with $T\geq 2$, and $\gamma_t \leq \nicefrac{1}{L}$ for all $t \in \{0,\dots,T-1\}$. Let $\lambda$ be as defined in Lemma~\ref{lemma:TM}. Then, for all $t$, the following holds true
\begin{align*}
\expect{\avgloss(\model{}{t+1}) - \avgloss(\model{}{t})}
    &\leq -\frac{\learningrate{t}}{2}\expect{\norm{\nabla  \avgloss(\model{}{t})}^2} +{\learningrate{t}}\frac{\lambda}{n-f}\sum_{i\in\H}\expect{{\norm{\mmt{i}{t}-\AvgMmt{t}}^2}}\\
 &+ {\learningrate{t}}\expect{\norm{ \nabla  \avgloss(\model{}{t})- \AvgMmt{t}}^2} \enspace.
\end{align*}
\end{replemma}

\begin{proof}
Consider an arbitrary step $t$. Note that Assumption~\ref{asp:lip} implies the Lipschitz continuity of $\nabla \avgloss(\model{}{})$ with coefficient $L$. Thus, we have 
\begin{align*}
    \avgloss(\model{}{t+1}) - \avgloss(\model{}{t})  &\leq \iprod{\model{}{t+1} - \model{}{t}}{\nabla  \avgloss(\model{}{t})} + \frac{L}{2} \norm{\model{}{t+1} - \model{}{t}}^2\enspace.
\end{align*}
Substituting from Algorithm~\ref{algorithm:dsgd}, $\model{}{t+1} = \model{}{t} - \learningrate{t} R_t$, we obtain that
\begin{align*}
    \avgloss(\model{}{t+1}) - \avgloss(\model{}{t})  \leq -\learningrate{t}\iprod{R_t}{\nabla  \avgloss(\model{}{t})} + \frac{L\learningrate{t}^2}{2} \norm{R_t}^2\enspace.
\end{align*}
Using the fact that $2\iprod{a}{b} = \norm{a}^2 + \norm{b}^2- \norm{a-b}^2$, we obtain that
\begin{align*}
    \avgloss(\model{}{t+1}) - \avgloss(\model{}{t})  &\leq -\frac{\learningrate{t}}{2}\norm{R_t}^2-\frac{\learningrate{t}}{2}\norm{\nabla  \avgloss(\model{}{t})}^2 +\frac{\learningrate{t}}{2}\norm{R_t - \nabla  \avgloss(\model{}{t})}^2 + \frac{L\learningrate{t}^2}{2} \norm{R_t}^2\\
    &= \left(\frac{L\learningrate{t}^2}{2}-\frac{\learningrate{t}}{2}\right) \norm{R_t}^2-\frac{\learningrate{t}}{2}\norm{\nabla  \avgloss(\model{}{t})}^2 +\frac{\learningrate{t}}{2}\norm{R_t - \nabla  \avgloss(\model{}{t})}^2\enspace.
\end{align*}
As $\learningrate{t} \leq 1/L$, we obtain that
\begin{align*}
    \avgloss(\model{}{t+1}) - \avgloss(\model{}{t})  &\leq -\frac{\learningrate{t}}{2}\norm{\nabla  \avgloss(\model{}{t})}^2 +\frac{\learningrate{t}}{2}\norm{R_t - \nabla  \avgloss(\model{}{t})}^2\\
    &\leq -\frac{\learningrate{t}}{2}\norm{\nabla  \avgloss(\model{}{t})}^2 +{\learningrate{t}}\norm{R_t - \AvgMmt{t}}^2 + {\learningrate{t}}\norm{ \nabla  \avgloss(\model{}{t})- \AvgMmt{t}}^2\enspace.
\end{align*}
Using Lemma~\ref{lemma:TM}, we then obtain that 
\begin{align*}
    \avgloss(\model{}{t+1}) - \avgloss(\model{}{t}) 
    &\leq -\frac{\learningrate{t}}{2}\norm{\nabla  \avgloss(\model{}{t})}^2 +{\learningrate{t}}\frac{\lambda}{n-f}\sum_{i\in\H}{{\norm{\mmt{i}{t}-\AvgMmt{t}}^2}}
 + {\learningrate{t}}\norm{ \nabla  \avgloss(\model{}{t})- \AvgMmt{t}}^2\enspace.
\end{align*}
Taking the total expectation from both sides we then have
\begin{align*}
\expect{\avgloss(\model{}{t+1}) - \avgloss(\model{}{t})}
    &\leq -\frac{\learningrate{t}}{2}\expect{\norm{\nabla  \avgloss(\model{}{t})}^2} +{\learningrate{t}}\frac{\lambda}{n-f}\sum_{i\in\H}\expect{{\norm{\mmt{i}{t}-\AvgMmt{t}}^2}}\\
 &+ {\learningrate{t}}\expect{\norm{ \nabla  \avgloss(\model{}{t})- \AvgMmt{t}}^2}\enspace,
\end{align*}
which is the desired result.
\end{proof}

\begin{replemma}{lemma:drift}
    Suppose assumptions~\ref{asp:bnd_var}, and~\ref{asp:heter} hold true, and consider Algorithm~\ref{algorithm:dsgd} with $T\geq 2$. Then, for any $t \in \{0,\dots, T-1\}$, the following holds true
    \begin{align*}
        &\frac{1}{n-f}\sum_{i\in\H}\expect{\norm{\Delta \mmt{i}{t}}^2}  \leq \beta_{t} \frac{1}{n-f}\sum_{i\in\H}\expect{\norm{\Delta m^{(i)}_{t-1} }^2}
        +(1-\beta_{t})\heter^2+ (1-\beta_{t})^2  \sigma^2 \enspace.
    \end{align*}
\end{replemma}

\begin{proof}
    Consider two arbitrary correct nodes $i$ and $j$. By the definition of the momentum vector from~\eqref{eqn:mmt_i}, we obtain that
    \begin{align*}
        \mmt{i}{t} - \mmt{j}{t} &=  \beta_t (\mmt{i}{t-1} - \mmt{j}{t-1}) + (1 - \beta_t) (\gradient{i}{t} - \gradient{j}{t})\\
        & = \beta_t (\mmt{i}{t-1} - \mmt{j}{t-1}) + (1 - \beta_t) \left(\nabla \localloss{i} (\model{}{t}) - \nabla \localloss{j} (\model{}{t}) \right) \\
        &+ (1 - \beta_t)\left(\gradient{i}{t} -\nabla \localloss{i}(\model{}{t}) - \gradient{j}{t}+\nabla \localloss{j}(\model{}{t})\right)\enspace.
    \end{align*}
    Taking the squared norm from both sides, we obtain that 
    \begin{align*}
        &\norm{\mmt{i}{t} - \mmt{j}{t}}^2  = \norm{\beta_t (\mmt{i}{t-1} - \mmt{j}{t-1}) + (1 - \beta_t) \left(\nabla \localloss{i} (\model{}{t}) - \nabla \localloss{j} (\model{}{t}) \right)}^2\\ &+ \norm{(1 - \beta_t)\left(\gradient{i}{t} -\nabla \localloss{i}(\model{}{t}) - \gradient{j}{t}+\nabla \localloss{j}(\model{}{t})\right)}^2 \\ &+ \iprod{\beta_t (\mmt{i}{t-1} - \mmt{j}{t-1}) + (1 - \beta_t) \left(\nabla \localloss{i} (\model{}{t}) - \nabla \localloss{j} (\model{}{t}) \right)}{(1 - \beta_t)\left(\gradient{i}{t} -\nabla \localloss{i}(\model{}{t}) - \gradient{j}{t}+\nabla \localloss{j}(\model{}{t})\right)}\enspace.
    \end{align*}
    Taking the conditional expectation $\condexpect{t}{\cdot}$ from both sides and noting that $\condexpect{t}{\gradient{i}{t}} = \nabla \localloss{i}(\model{}{t})$ and $ \condexpect{t}{ \gradient{j}{t}} = \nabla \localloss{j}(\model{}{t})$, we obtain that 
    \begin{align*}
        \condexpect{t}{\norm{\mmt{i}{t} - \mmt{j}{t}}^2} &= \norm{\beta_t (\mmt{i}{t-1} - \mmt{j}{t-1}) + (1 - \beta_t) \left(\nabla \localloss{i} (\model{}{t}) - \nabla \localloss{j} (\model{}{t}) \right)}^2\\ &+ (1-\beta_t)^2 \condexpect{t}{\norm{\gradient{i}{t}- \nabla \localloss{i}(\model{}{t})}^2} + (1-\beta_t)^2 \condexpect{t}{\norm{\gradient{j}{t}- \nabla \localloss{j}(\model{}{t})}^2}\enspace.
    \end{align*}
    Using Assumption~\ref{asp:bnd_var}, we then obtain that 
    \begin{align*}
        \condexpect{t}{\norm{\mmt{i}{t} - \mmt{j}{t}}^2} &\leq \norm{\beta_t (\mmt{i}{t-1} - \mmt{j}{t-1}) + (1 - \beta_t) \left(\nabla \localloss{i} (\model{}{t}) - \nabla \localloss{j} (\model{}{t}) \right)}^2 + 2(1-\beta_t)^2 \sigma^2\enspace.
    \end{align*}
    By Jensen's inequality, we then have
    \begin{align*}
        \condexpect{t}{\norm{\mmt{i}{t} - \mmt{j}{t}}^2} &\leq \beta_t \norm{ \mmt{i}{t-1} - \mmt{j}{t-1}}^2 + (1 - \beta_t) \norm{\nabla \localloss{i} (\model{}{t}) - \nabla \localloss{j} (\model{}{t})}^2 + 2 (1-\beta_t)^2  \sigma^2\enspace.
    \end{align*}
    Taking total expectation and averaging over all possible $i,j \in \H$, we then obtain that
    \begin{align*}
        &\frac{1}{(n-f)^2}\sum_{i,j\in\H}\expect{\norm{\mmt{i}{t}-\mmt{j}{t}}^2} \leq \beta_t \frac{1}{(n-f)^2}\sum_{i,j\in\H}\expect{\norm{\mmt{i}{t-1}-\mmt{j}{t-1}}^2}\\
        &+(1-\beta_t)\frac{1}{(n-f)^2}\sum_{i,j\in\H}\expect{ \norm{\nabla \localloss{i} (\model{}{t}) - \nabla \localloss{j} (\model{}{t})}^2} + 2 (1-\beta_t)^2  \sigma^2\enspace.
    \end{align*}
    Using Lemma~\ref{lem:diam_equal}, we then obtain that 
    \begin{align*}
        &\frac{1}{n-f}\sum_{i\in\H}\expect{\norm{\mmt{i}{t}-\AvgMmt{t}}^2} \leq \beta_t \frac{1}{n-f}\sum_{i\in\H}\expect{\norm{\mmt{i}{t-1}-\AvgMmt{t-1}}^2}\\
        &+(1-\beta_t)\frac{1}{n-f}\sum_{i\in\H}\expect{ \norm{\nabla \localloss{i} (\model{}{t}) - \nabla \avgloss (\model{}{t})}^2} +  (1-\beta_t)^2  \sigma^2\enspace.
    \end{align*}
    By Assumption~\ref{asp:heter}, we then obtain that
    \begin{align*}
        &\frac{1}{n-f}\sum_{i\in\H}\expect{\norm{\mmt{i}{t}-\AvgMmt{t}}^2} \leq \beta_t \frac{1}{n-f}\sum_{i\in\H}\expect{\norm{\mmt{i}{t-1}-\AvgMmt{t-1}}^2}
        +(1-\beta_t)\heter^2+  (1-\beta_t)^2  \sigma^2\enspace.
    \end{align*}
    This is the desired result.
\end{proof}

\begin{replemma}{lemma:dev}
Suppose assumptions~\ref{asp:lip},~\ref{asp:bnd_var}, and~\ref{asp:heter} hold true, consider Algorithm~\ref{algorithm:dsgd} with $T\geq 2$, and $\lambda$ as defined in Lemma~\ref{lemma:TM}. Then for any $t \in \{0,\dots,T-1\}$, the following holds true
\begin{align*}
    \expect{\norm{\dev{t+1}}^2} & \leq 
      \beta_{t+1}^2 \left(1+4\learningrate{t} L+3\learningrate{t}^2 L^2 \right )\expect{\norm{\dev{t}}^2} + (1-\beta_{t+1})^2\frac{\sigma^2}{n-f} \\
        &+3\beta_{t+1}^2(\learningrate{t}^2 L^2+\learningrate{t} L)\left(  \frac{\lambda}{n-f}\sum_{i\in\H}{\expect{\norm{\Delta \mmt{i}{t}}^2}} +  \expect{\norm{\nabla\avgloss (\model{}{t})}^2} \right) \enspace.
\end{align*}
\end{replemma}

\begin{proof}
    We recall that at any round $t \in \{0,\dots,T-1\}$ and any worker $i\in \mathcal{H}$, the momentum $m_t^{(i)}$ is computed as follows

    \begin{align*}
        \mmt{i}{t}= \beta_t \mmt{i}{t-1} + (1 - \beta_t) \gradient{i}{t}\enspace.
    \end{align*}
    Hence, we have
    \begin{align*}
        \dev{t+1} = \AvgMmt{t+1} - \nabla\avgloss (\model{}{t+1}) = \beta_{t+1} \AvgMmt{t} + (1 - \beta_{t+1}) \avggrad{t+1} - \nabla\avgloss (\model{}{t+1})\enspace.
    \end{align*}
    where for any $t \in \{0,\dots,T-1\}$, $\AvgMmt{t} \coloneqq \frac{1}{(n-f)} \sum_{i \in \H} \mmt{i}{t}$ and $\avggrad{t} \coloneqq \frac{1}{(n-f)} \sum_{i \in \H} \gradient{i}{t}$.
    
    Adding and subtracting $\beta_{t+1}\nabla\avgloss (\model{}{t})$, we obtain that

    \begin{align*}
        \dev{t+1} &=  \beta_{t+1} \left(\AvgMmt{t} - \nabla\avgloss (\model{}{t}) \right)+ (1 - \beta_{t+1}) \avggrad{t+1} - \nabla\avgloss (\model{}{t+1}) + \beta_{t+1}\nabla\avgloss (\model{}{t})\\
        &= \beta_{t+1} \left(\AvgMmt{t} - \nabla\avgloss (\model{}{t}) \right)+ (1 - \beta_{t+1}) \left(\avggrad{t+1} - \nabla\avgloss (\model{}{t+1})\right) + \beta_{t+1}\left(\nabla\avgloss (\model{}{t})- \nabla\avgloss (\model{}{t+1}) \right)\enspace.
    \end{align*}
    Now by Assumption~\ref{asp:bnd_var}, we have $\condexpect{t+1}{\avggrad{t+1}} =  \nabla\avgloss (\model{}{t+1})$ and $\condexpect{t+1}{\norm{\avggrad{t+1} -\nabla\avgloss (\model{}{t+1})}^2} \leq \frac{\sigma^2}{n-f}$. Therefore,
    \begin{align*}
        \condexpect{t+1}{\norm{\dev{t+1}}^2} \leq \beta_{t+1}^2 \norm{\dev{t}+\nabla\avgloss (\model{}{t})- \nabla\avgloss (\model{}{t+1})}^2 + (1-\beta_{t+1})^2\frac{\sigma^2}{n-f}\enspace.
    \end{align*}
    Now as $(a+b)^2 \leq (1+c)a^2 
    + (1+1/c)b^2$ for any $c$, we obtain that
    \begin{align*}
        \condexpect{t+1}{\norm{\dev{t+1}}^2} &\leq \beta_{t+1}^2(1+\learningrate{t} L)\norm{\dev{t}}^2 + \beta_{t+1}^2(1+\frac{1}{\learningrate{t} L})\norm{\nabla\avgloss (\model{}{t})- \nabla\avgloss (\model{}{t+1})}^2 \\&+ (1-\beta_{t+1})^2\frac{\sigma^2}{n-f}\enspace.
    \end{align*}
    From Assumption~\ref{asp:lip}, we have $\norm{\nabla\avgloss (\model{}{t})- \nabla\avgloss (\model{}{t+1})} \leq L \norm{\model{}{t} - \model{}{t+1}}$. Using this above, we obtain that 
    \begin{align}
        \condexpect{t+1}{\norm{\dev{t+1}}^2} &\leq \beta_{t+1}^2(1+\learningrate{t} L)\norm{\dev{t}}^2 + \beta_{t+1}^2\left(1+\frac{1}{\learningrate{t} L}\right)L^2\norm{ \model{}{t}- \model{}{t+1}}^2 \notag \\
        & \ \ \ \ \ + (1-\beta_{t+1})^2\frac{\sigma^2}{n-f}\enspace. \label{eq:devint}
    \end{align}
    Now recall that $ \model{}{t}- \model{}{t+1} = \learningrate{t} R_t$. Therefore,
    \begin{align*}
        \norm{\model{}{t}- \model{}{t+1}}^2 &= \learningrate{t}^2 \norm{ R_t}^2 \\
        &= \learningrate{t}^2 \norm{ R_t- \AvgMmt{t} + \AvgMmt{t} - \nabla\avgloss (\model{}{t}) + \nabla\avgloss (\model{}{t})}^2\\
        &\leq 3\learningrate{t}^2 \norm{ R_t- \AvgMmt{t}}^2 + 3\learningrate{t}^2 \norm{\AvgMmt{t} - \nabla\avgloss (\model{}{t})}^2 + 3\learningrate{t}^2 \norm{\nabla\avgloss (\model{}{t})}^2\\
        &\leq 3\learningrate{t}^2 \frac{\lambda}{n-f}\sum_{i\in\H}{\norm{\mmt{i}{t}-\AvgMmt{t}}^2}+ 3\learningrate{t}^2 \norm{\AvgMmt{t} - \nabla\avgloss (\model{}{t})}^2 + 3\learningrate{t}^2 \norm{\nabla\avgloss (\model{}{t})}^2\enspace,
    \end{align*}
    where in the last inequality we used ~\ref{lemma:TM}.
    Combining this with~\eqref{eq:devint}, we obtain that
    \begin{align*}
        &\condexpect{t+1}{\norm{\dev{t+1}}^2} \leq \beta_{t+1}^2(1+\learningrate{t} L)\norm{\dev{t}}^2 + (1-\beta_{t+1})^2\frac{\sigma^2}{n-f} \\
        &+\beta_{t+1}^2(1+\frac{1}{\learningrate{t} L})L^2 \left( 3\learningrate{t}^2 \frac{\lambda}{n-f}\sum_{i\in\H}{\norm{\mmt{i}{t}-\AvgMmt{t}}^2}+ 3\learningrate{t}^2 \norm{\dev{t}}^2 + 3\learningrate{t}^2 \norm{\nabla\avgloss (\model{}{t})}^2 \right)\enspace.
    \end{align*}
    Rearranging the terms and taking the total expectation, we obtain that 
    \begin{align*}
        &\expect{\norm{\dev{t+1}}^2} \leq \beta_{t+1}^2(1+4\learningrate{t} L+3\learningrate{t}^2 L^2)\expect{\norm{\dev{t}}^2} + (1-\beta_{t+1})^2\frac{\sigma^2}{n-f} \\
        &+3\beta_{t+1}^2(\learningrate{t}^2 L^2+\learningrate{t} L)\left(  \frac{\lambda}{n-f}\sum_{i\in\H}{\expect{\norm{\mmt{i}{t}-\AvgMmt{t}}^2}} +  \expect{\norm{\nabla\avgloss (\model{}{t})}^2} \right)\enspace.
    \end{align*}
    This is the desired result.
\end{proof}

\begin{replemma}{lemma:lyap}
Suppose assumptions~\ref{asp:lip},~\ref{asp:polyak},~\ref{asp:bnd_var}, and~\ref{asp:heter} hold true. Consider Algorithm~\ref{algorithm:dsgd} with $T\geq 2$ and a set of parameters such that $t \in \{0, \cdots ,T\}$, $\learningrate{t} \leq \frac{1}{18L}$, and $1 - \beta_{t+1} =  18 \learningrate{t} L$. Finally, let $(\lyap{t})_{t \geq 0}$ be as defined in~\eqref{eq:lyap_def} and $\lambda$ as defined in Lemma~\ref{lemma:TM}. Then the following holds true
\begin{align*}
\lyap{t+1}
    &\leq  \left( 1-\frac{\mu \learningrate{t}}{3} \right) \lyap{t} + 27 L \left(\lambda+\frac{1}{n-f} \right) \sigma^2 \learningrate{t}^2 +  \frac{3}{2}  \lambda \heter^2 \learningrate{t} \enspace.
\end{align*}
\end{replemma}

\begin{proof}
Consider an arbitrary $t \in \{0,\dots,T-1\}$. Combining Lemmas~\ref{lemma:drift}, ~\ref{lemma:dev}, and~\ref{lemma:loss_bound}, we obtain that
\begin{align*}
    \lyap{t+1}  &=  \expect{\avgloss\left( \model{}{t+1} \right)} - \optloss + \rho \expect{\norm{\dev{t+1}}^2}+ \rho \frac{\lambda}{n-f}\sum_{i\in\H}\expect{\norm{\mmt{i}{t+1}-\AvgMmt{t+1}}^2} \\
    &\leq \expect{\avgloss\left( \model{}{t} \right)} - \optloss -\frac{\learningrate{t}}{2}\expect{\norm{\nabla  \avgloss(\model{}{t})}^2} +{\learningrate{t}}\frac{\lambda}{n-f}\sum_{i\in\H}\expect{{\norm{\mmt{i}{t}-\AvgMmt{t}}^2}}\\
 &+ {\learningrate{t}}\expect{\norm{ \dev{t}}^2} + \rho\beta_{t+1}^2(1+4\learningrate{t} L+3\learningrate{t}^2 L^2)\expect{\norm{\dev{t}}^2} + \rho(1-\beta_{t+1})^2\frac{\sigma^2}{n-f} \\
        &+3\rho\beta_{t+1}^2(\learningrate{t}^2 L^2+\learningrate{t} L)\left(  \frac{\lambda}{n-f}\sum_{i\in\H}{\expect{\norm{\mmt{i}{t}-\AvgMmt{t}}^2}} +  \expect{\norm{\nabla\avgloss (\model{}{t})}^2} \right) \\ &+ \rho \lambda\beta_{t+1} \frac{1}{n-f}\sum_{i\in\H}\expect{\norm{\mmt{i}{t}-\AvgMmt{t}}^2}
        +\rho \lambda(1-\beta_{t+1})\heter^2+ \rho \lambda(1-\beta_{t+1})^2  \sigma^2\enspace.
\end{align*}
Re-arranging the terms, we obtain that
\begin{align}\nonumber
    \lyap{t+1}
    &\leq \expect{\avgloss\left( \model{}{t} \right)} - \optloss + \left( -\frac{\learningrate{t}}{2} +3\rho\beta_{t+1}^2(\learningrate{t}^2 L^2+\learningrate{t} L) \right)\expect{\norm{\nabla  \avgloss(\model{}{t})}^2}\\\nonumber
    &+ \left({\learningrate{t}} +3\rho\beta_{t+1}^2(\learningrate{t}^2 L^2+\learningrate{t} L) + \rho \beta_{t+1}\right) \frac{\lambda}{n-f}\sum_{i\in\H}\expect{{\norm{\mmt{i}{t}-\AvgMmt{t}}^2}}\\ \nonumber
    &+ \left(\learningrate{t}+ \rho\beta_{t+1}^2(1+4\learningrate{t} L+3\learningrate{t}^2 L^2)\right)\expect{\norm{\dev{t}}^2}  \\ 
    &+\rho \lambda(1-\beta_{t+1})\heter^2+ \rho \lambda(1-\beta_{t+1})^2  \sigma^2 + \rho(1-\beta_{t+1})^2\frac{\sigma^2}{n-f}\enspace.
    \label{eq:beforesimplify}
\end{align}
We denote,
\begin{align*}
    A &\coloneqq -\frac{\learningrate{t}}{2} +3\rho\beta_{t+1}^2(\learningrate{t}^2 L^2+\learningrate{t} L)  , \\
    B &\coloneqq {\learningrate{t}} +3\rho\beta_{t+1}^2(\learningrate{t}^2 L^2+\learningrate{t} L) + \rho \beta_{t+1},\\
    C &\coloneqq \learningrate{t}+ \rho\beta_{t+1}^2(1+4\learningrate{t} L+3\learningrate{t}^2 L^2) \\ \nonumber
    D &\coloneqq \rho \lambda(1-\beta_{t+1})\heter^2+ \rho \lambda(1-\beta_{t+1})^2  \sigma^2 + \rho(1-\beta_{t+1})^2\frac{\sigma^2}{n-f}\enspace.
\end{align*}
Substituting from above in~\eqref{eq:beforesimplify} we obtain that
\begin{align*}
    \lyap{t+1}
    &\leq \expect{\avgloss\left( \model{}{t} \right)} - \optloss + A\expect{\norm{\nabla  \avgloss(\model{}{t})}^2}\\
    &+B\frac{\lambda}{n-f}\sum_{i\in\H}\expect{{\norm{\mmt{i}{t}-\AvgMmt{t}}^2}}+ C\expect{\norm{\dev{t}}^2} + D\enspace.
\end{align*}
Now, we separately analyse the terms $A$, $B$, $C$ and $D$ below by using the following,
\begin{align}
    \rho = \frac{1}{12 L}, ~ \learningrate{t} \leq \frac{1}{18L}, ~ \text{ and } ~ 1 - \beta_{t+1} =  18 \learningrate{t} L\enspace. 
    \label{eq:conditions}
\end{align}
Note that the condition on $\learningrate{t}$ above follows 

\noindent{\bf Term A.} Using the facts that $\rho = 1/12L$, $\learningrate{t} \leq 1/18L \leq 1/3L$ and that $\beta_{t+1}^2 < 1$, we obtain that
\begin{align}\nonumber
    A = -\frac{\learningrate{t}}{2} +3\rho\beta_{t+1}^2(\learningrate{t}^2 L^2+\learningrate{t} L) &\leq  -\frac{\learningrate{t}}{2} +3\rho(\learningrate{t}^2 L^2+\learningrate{t} L) \\
    &\leq -\frac{\learningrate{t}}{2} +\frac{1}{4L}(\frac{\learningrate{t} L}{3}+\learningrate{t} L) = -\frac{\learningrate{t}}{6}\enspace.
    \label{eqn:analyse_A}
\end{align}

{\bf Term B.} We obtain that
\begin{align*}
    B = {\learningrate{t}} +3\rho\beta_{t+1}^2(\learningrate{t}^2 L^2+\learningrate{t} L) + \rho \beta_{t+1} = \rho (12 \learningrate{t} L +3\beta_{t+1}^2(\learningrate{t}^2 L^2+\learningrate{t} L)+ \beta_{t+1})\enspace.
\end{align*}
Noting that $\beta_{t+1} \leq 1$, $\beta_{t+1} = 1 - 18 \learningrate{t}L$ and $\learningrate{t} \leq 1/18L \leq 1/12L$ we obtain that
\begin{align*}
    B \leq \rho \left(12\learningrate{t} L+3\learningrate{t} L+\frac{\learningrate{t} L}{4} +  (1-18\learningrate{t} L) \right) \leq \rho \left( 1-\frac{11\learningrate{t} L}{4}  \right) \leq \rho \left( 1-\frac{\learningrate{t} L}{3}  \right) \leq \rho \left( 1-\frac{\mu \learningrate{t}}{3} \right)\enspace,
\end{align*}
where in the last inequality we used $\mu\leq L$.

\noindent{\bf Term C.} Using the facts that $\beta_{t+1} <1$ and $\rho = 1/12L$, we obtain that
\begin{align*}
    C \coloneqq \learningrate{t}+ \rho\beta_{t+1}^2(1+4\learningrate{t} L+3\learningrate{t}^2 L^2) &\leq  \rho \left(\frac{\learningrate{t}}{\rho} + \beta_{t+1} + 4\learningrate{t} L+3\learningrate{t}^2 L^2 \right)\\ &= \rho \left(12{\learningrate{t}L} + \beta_{t+1} + 4\learningrate{t} L+3\learningrate{t}^2 L^2 \right)\enspace.
\end{align*}
Using the fact $\learningrate{t} \leq 1/18L \leq 1/12L$ we then have
\begin{align}
    C &\leq \rho \left( 16\learningrate{t} L + \frac{\learningrate{t} L}{4} + (1 - 18\learningrate{t} L) \right) \leq \rho \left( 1-\frac{7\learningrate{t} L}{4}  \right) \leq \rho \left( 1-\frac{\learningrate{t} L}{3}  \right) \leq \rho \left( 1-\frac{\mu \learningrate{t}}{3} \right)\enspace,
\end{align}

\noindent{\bf Term D.}
\begin{align*}
    D &= \rho \lambda(1-\beta_{t+1})\heter^2+ \rho \lambda(1-\beta_{t+1})^2  \sigma^2 + \rho(1-\beta_{t+1})^2\frac{\sigma^2}{n-f}\\
    &= \frac{3}{2}\learningrate{t} \lambda \heter^2 + 27 \learningrate{t}^2 L \left(\lambda+\frac{1}{n-f} \right) \sigma^2\enspace.
\end{align*}
Combining all, we obtain that
\begin{align*}
    \lyap{t+1}
    &\leq \expect{\avgloss\left( \model{}{t} \right)} - \optloss -\frac{\learningrate{t}}{6}\expect{\norm{\nabla  \avgloss(\model{}{t})}^2}+ \left( 1-\frac{\mu \learningrate{t}}{3} \right)\rho\expect{\norm{\dev{t}}^2} \\
    &+ \left( 1-\frac{\mu \learningrate{t}}{3} \right)\rho\frac{\lambda}{n-f}\sum_{i\in\H}\expect{{\norm{\mmt{i}{t}-\AvgMmt{t}}^2}}+  \frac{3}{2}\learningrate{t}  \lambda  \heter^2 + 27 \learningrate{t}^2 L \left(\lambda+\frac{1}{n-f} \right) \sigma^2\enspace.
\end{align*}
Recall from Assumption~\ref{asp:polyak} that $\norm{\nabla  \avgloss(\model{}{t})}^2 \geq 2\mu\left(  \avgloss\left( \model{}{t} \right) - \optloss \right) $. Therefore,

\begin{align*}
     \lyap{t+1}
    &\leq \left( 1-\frac{\mu \learningrate{t}}{3} \right) \left(\expect{\avgloss\left( \model{}{t} \right)} - \optloss + \rho \expect{ \norm{\dev{t}}^2}+ \rho \frac{\lambda}{n-f}\sum_{i\in\H}\expect{\norm{\mmt{i}{t}-\AvgMmt{t}}^2}\right) \\
    &+  \frac{3}{2}\learningrate{t}  \lambda  \heter^2 + 27 \learningrate{t}^2 L \left(\lambda+\frac{1}{n-f} \right) \sigma^2\\
    & = \left( 1-\frac{\mu \learningrate{t}}{3} \right) \lyap{t} + 27 L \left(\lambda+\frac{1}{n-f} \right) \sigma^2 \learningrate{t}^2 +  \frac{3}{2}  \lambda \heter^2 \learningrate{t}\enspace.
\end{align*}
\end{proof}

\begin{replemma}{lem:stepsize_strongly_convex}
    Let $a,b,c,d$ be positive real values with $a < b$, and let $T\geq 2$ be a positive integer. Let $(\learningrate{0}, \dots, \learningrate{T-1})$ and $(r_0, \dots, r_{T})$ be real valued sequences such that for all $t \in \{0,\dots,T-1\}$,
    \begin{align}
        r_{t+1} \leq (1 - a \learningrate{t}) r_{t} + c \learningrate{t}^2 + d \learningrate{t} \enspace. \label{eqn:rt-rt}
    \end{align}
    Consider the following two cases: 
    \begin{itemize}
        \item \textbf{Case 1:} $T \leq \nicefrac{b}{a} ~ $ and $\gamma_t = \nicefrac{1}{b}, ~ \forall t \in \{0,\dots,T-1\}$.
        \item \textbf{Case 2:} $T >  \nicefrac{b}{a} ~ $ and for $s = \nicefrac{2b}{a} ~ $ and $t_0 = \ceil{\nicefrac{T}{2}}$,
        \begin{align*}
            \learningrate{t} = \left\{
            \begin{array}{ccc}
                \frac{1}{b} & , &\text{ \quad if } t < t_0 \\ 
                ~ \\
                \frac{2}{a(s+t-t_0+1)} & , & \text{ otherwise }
            \end{array} \right. \enspace.
        \end{align*}
    \end{itemize}
    In both Case 1 and Case 2, we have
    \begin{align}
        r_T \leq  r_0 \exp \left( - \frac{aT}{2b} \right)  + \frac{18c}{a^2 T}  + \frac{3d}{a} \enspace. \label{eqn:rT-r0}
    \end{align}
\end{replemma}

\begin{proof}
    Our technique closely follows that of the proof of Lemma 3 in \cite{khaled20}, which itself build upon the analysis presented in~\cite{stich19}.\\

    \noindent{\bf Case 1.} Here, $T\leq \frac{b}{a}$ and $\gamma_t = \gamma = \nicefrac{1}{b}, ~ \forall t \in \{0,\dots,T-1\}$.
    Thus, as $a < b$, note that $(1 - a \gamma) \in (0, \, 1)$. Then, by applying recursion on~\eqref{eqn:rt-rt} we obtain that for all $t \in [T]$,
    \begin{align*}
        r_{t + 1} &\leq (1 - a\learningrate{})^{t+1} r_0 + \learningrate{}^2 c \sum_{\tau = 0}^{t} (1 - a\learningrate{})^\tau + \learningrate{} d \sum_{\tau = 0}^{t} (1 - a\learningrate{})^\tau \leq  (1 - a\learningrate{})^{t + 1} r_0 + \frac{\learningrate{} c}{a}+ \frac{d}{a}\enspace,
    \end{align*}
    where the last inequation comes from the fact that $\sum_{\tau = 0}^{t} (1 - a\learningrate{})^\tau \leq \sum_{\tau = 0}^{\infty} (1 - a\learningrate{})^\tau = \frac{1}{1 - (1 - a\learningrate{})}$.
    As $(1-x) \leq \exp{(-x)}$ for all $x \geq 0$, the above implies that for all $t \in [T]$,
    \begin{align*}
        r_{t + 1} \leq r_0 \exp\left(-a\learningrate{} (t + 1) \right) + \frac{\learningrate{} c}{a}+ \frac{d}{a}\enspace.
    \end{align*}
    Substituting $\learningrate{} = \nicefrac{1}{b}$ in the above, we obtain that for all $t \in \{0,\dots,T-1\}$,
    \begin{align}
        r_{t + 1} &\leq  r_0 \exp\left(-\frac{a (t + 1)}{b} \right) + \frac{c}{ab}+ \frac{d}{a}\enspace. \label{eq:r0_interm}
    \end{align}
    Recall that in this particular case, we assume $T a \leq b$. Thus, $\frac{1}{b} \leq \frac{1}{Ta}$ and we obtain that for all $t \in \{0,\dots,T-1\}$,
    \begin{align*}
        r_{t + 1} &\leq r_0 \exp\left(-\frac{a (t + 1)}{b} \right) + \frac{c}{a^2 T}+ \frac{d}{a}\enspace.
    \end{align*}
    Substituting $t = (T - 1)$ in the above yields 
    \begin{align*}
        r_T &\leq r_0 \exp\left(-\frac{a T}{b} \right) + \frac{c}{a^2 T}+ \frac{d}{a}\enspace.
    \end{align*}
    As $T > T/2$ and $a,c,d >0$, we have, 
    \begin{align*}
        r_T &\leq r_0 \exp\left(-\frac{a T}{2b} \right) + \frac{18c}{a^2 T}+ \frac{3d}{a}\enspace.
    \end{align*}

    \noindent{\bf Case 2.} $T > \nicefrac{b}{a}$ and for $s = \nicefrac{2b}{a} ~ $ and $t_0 = \ceil{\nicefrac{T}{2}}$,
        \begin{align*}
            \learningrate{t} = \left\{
            \begin{array}{ccc}
                \frac{1}{b} & , &\text{ \quad if } t < t_0 \\ 
                ~ \\
                \frac{2}{a(s+t-t_0+1)} & , & \text{ otherwise }
            \end{array} \right. \enspace.
        \end{align*}
    First, we consider the sub-case when $t < t_0$. As $\learningrate{t} = \learningrate{} = \nicefrac{1}{b}$ for all $t < t_0$,~\eqref{eq:r0_interm} holds true for any $t < t_0$. Thus, upon substituting $t = t_0 - 1$ in~\eqref{eq:r0_interm} we obtain that
    \begin{align*}
        r_{t_0} \leq  r_0 \exp\left(-\frac{a t_0}{b} \right) + \frac{c}{ab}+ \frac{d}{a}\enspace. 
    \end{align*}
    As $t_0 \geq \frac{T}{2}$, the above implies that
    \begin{align}\label{eq:rto}
        r_{t_0} &\leq  r_0 \exp\left( -\frac{aT}{2b} \right) + \frac{c}{ab}+ \frac{d}{a}\enspace.
    \end{align}
    Next, we consider the sub-case when $t_0 \leq t \leq T-1$. For an arbitrary such $t$, upon substituting $\learningrate{t} = \frac{2}{a(s+t-t_0+1)}$ in~\eqref{eqn:rt-rt} we obtain that
    \begin{align*}
        r_{t+1} &\leq (1 - a \learningrate{t}) r_{t} + c \, \learningrate{t}^2 + d \learningrate{t} = \left(1 -  \frac{2}{s+t - t_0 + 1} \right) r_{t} + \frac{4c}{a^2(s+t-t_0+1)^2} + \frac{2 d}{a(s+t-t_0+1)}\\
        &= \left(\frac{s + t- t_0 - 1}{s+t - t_0 + 1} \right) r_{t} + \frac{4c}{a^2(s+t-t_0+1)^2} + \frac{2 d}{a(s+t-t_0+1)}\enspace.
    \end{align*}
    Multiplying both sides above by $(s+t- t_0 + 1)^2$ we obtain that
    \begin{align*}
        (s+t- t_0 + 1)^2r_{t+1} &\leq (s+t-t_0-1)(s+t-t_0+1)  r_t  + \frac{4c}{a^2} + \frac{2 d}{a} (s+t- t_0 + 1) \\
        &= ((s+t-t_0)^2-1) r_t  + \frac{4c}{a^2} + \frac{2 d}{a} (s+t- t_0 + 1)\\
        &\leq (s+t-t_0)^2 r_t  + \frac{4c}{a^2} + \frac{2 d}{a} (s+t- t_0 + 1)\enspace.
    \end{align*}
    By rewriting $(s+t- t_0 + 1)$ as $(s+t + 1 - t_0)$ in the above, we have
    \begin{align*}
        (s+t + 1 - t_0 )^2r_{t+1} \leq (s+t-t_0)^2 r_t  + \frac{4c}{a^2} + \frac{2 d}{a} (s+t + 1 - t_0)\enspace.
    \end{align*}
    Recall that $t$ above is an arbitrary integer in $[t_0, \, T-1]$. Thus, the inequality holds true for all $t \in [t_0, \, T-1]$. Therefore, upon summing both the sides over all $t \in [t_0, \, T-1]$, we have
    \begin{align*}
        \sum_{t=t_0}^{T-1} (s+t + 1 - t_0 )^2r_{t+1} \leq \sum_{t=t_0}^{T-1} (s+t-t_0)^2 r_t  +  \sum_{t=t_0}^{T-1} \frac{4c}{a^2}   + \frac{2 d}{a} \sum_{t=t_0}^{T-1} (s+t + 1 - t_0)\enspace.
    \end{align*}
    Upon expanding the LHS and the first-term in the RHS we obtain that
    \begin{align*}
        (s+T-t_0)^2 \, r_{T} &\leq s^2 r_{t_0} + \sum_{t=t_0}^{T-1} \frac{4c}{a^2} + \frac{2d}{a} \, \sum_{t=t_0}^{T-1} (s+t+1-t_0)\\
        & =  s^2 r_{t_0} +  \frac{4c}{a^2} (T-t_0)+ \frac{d}{a} (T-t_0)(T-t_0+1+2s)\enspace.
    \end{align*}
    Therefore,
    \begin{align*}
        r_T \leq  \frac{s^2}{(s+T-t_0)^2} \, r_{t_0} +  \frac{4c (T-t_0)}{a^2 (s+T-t_0)^2} + \frac{d (T-t_0)(T-t_0+1+2s)}{a (s+T-t_0)^2}\enspace.
    \end{align*}
    As $ T-t_0 \leq s+T-t_0 $ and $T-t_0+1+2s \leq 2(s+T-t_0)$, from above we obtain that
    \begin{align*}
        r_T &\leq  \frac{s^2}{(s+T-t_0)^2} \, r_{t_0} +  \frac{4c}{a^2 (T - t_0)}  + \frac{2d}{a}\enspace.
    \end{align*}
    As $t_0 \leq \frac{2T}{3}$, we have $T - t_0 \geq \frac{T}{3}$. Using this above we obtain that
    \begin{align*}
        r_T \leq  \frac{s^2}{(s+T-t_0)^2} r_{t_0} +  \frac{12c}{a^2 T} + \frac{2d}{a}\enspace.
    \end{align*}
    Substituting from~\eqref{eq:rto} in the above, we obtain that
    \begin{align*}
        r_T
        & \leq  \frac{s^2}{(s+T-t_0)^2} \left(  r_0 \exp\left(-\frac{aT}{2b} \right)  + \frac{c}{ab}+ \frac{d}{a} \right) +  \frac{12c}{a^2T} + \frac{2d}{a}\enspace.
    \end{align*}
    As $s \leq s+T-t_0$, the above implies that
     \begin{align*}
        r_T
        & \leq  \frac{s}{s+T-t_0} \left(  \frac{c}{ab} \right) + r_0 \exp\left(-\frac{aT}{2b} \right) + \frac{d}{a}  +  \frac{12c}{a^2T} + \frac{2d}{a}\enspace.
    \end{align*}
    Using the fact that $s+T-t_0 \geq \frac{T}{3}$ above we have
    \begin{align*}
        r_T \leq  \frac{3s}{T} \left(  \frac{c}{ab} \right) + r_0 \exp\left(-\frac{aT}{2b} \right) + \frac{d}{a}  +  \frac{12c}{a^2T} + \frac{2d}{a}\enspace.
    \end{align*}
    Substituting $s = \frac{2b}{a}$ proves~\eqref{eqn:rT-r0}, i.e., we obtain that
    \begin{align*}
        r_T \leq  r_0 \exp\left(-\frac{aT}{2b} \right) + \frac{18c}{a^2 T}+ \frac{3d}{a}\enspace.
    \end{align*}
\end{proof}

\subsection{Final step to prove Theorem~\ref{thm:main}}

\begin{proof}[Proof of Theorem~\ref{thm:main}]
    We now apply Lemma~\ref{lem:stepsize_strongly_convex} to the recursion of Lemma~\ref{lemma:lyap}, for $a = \frac{\mu}{3}$, $b = {18L}$, $c = 27 L \left(\lambda+\frac{1}{n-f} \right) \sigma^2$ and $d =  \frac{3}{2}  \lambda  \heter^2$. Choosing the learning rates as specified in Lemma~\ref{lem:stepsize_strongly_convex}, we then obtain that
\begin{align}
\label{eq:lyapanovfoundbeforeV}
     \lyap{T} \leq \exp(-\frac{\mu T}{108 L}) \lyap{0} + \frac{4374L \left(\lambda+\frac{1}{n-f} \right) \sigma^2}{T\mu^2}+ \frac{9 \lambda\heter^2}{2\mu} \enspace.
\end{align}
As $\mmt{i}{0} = 0$ for all $i \in \H$, we have 
\begin{align*}
    \frac{1}{n-f}\sum_{i\in\H}\norm{\mmt{i}{0}-\AvgMmt{0}}^2 = 0 \enspace,
\end{align*}
and
\begin{align*}
    \norm{\dev{0}}^2 = \norm{ \nabla  \avgloss(\model{}{0})- \AvgMmt{0}}^2 =  \norm{ \nabla  \avgloss(\model{}{0})}^2 \leq 2L \left(\avgloss(\model{}{0}) - \optloss  \right) \enspace,
\end{align*}
where in the last inequality we used Lemma~\ref{lemma:smootnessbound}. Thus,
\begin{align*}
    \lyap{0} = {\avgloss\left( \model{}{0} \right) - \optloss + \frac{1}{12L}  \norm{\dev{0}}^2+ \frac{1}{12L} \frac{\lambda}{n-f}\sum_{i\in\H}\norm{\mmt{i}{0}-\AvgMmt{0}}^2} \leq \frac{7}{6} \left(\avgloss(\model{}{0}) - \optloss  \right)\enspace.
\end{align*}
Combining this with~\eqref{eq:lyapanovfoundbeforeV}, we obtain that
\begin{align*}
    \lyap{T} \leq \frac{7}{6} \left(\avgloss(\model{}{0}) - \optloss  \right) \cdot \exp(-\frac{\mu T}{108 L})+ \frac{4374L \left(\lambda+\frac{1}{n-f} \right) \sigma^2}{T\mu^2}+ \frac{9 \lambda\heter^2}{2\mu} \enspace.
\end{align*}
 By the definition of~$\lyap{t}$ in~\eqref{eq:lyap_def}, we have $\expect{\avgloss(\model{}{T}) - \optloss}  \leq \lyap{T}$. Therefore,
  \begin{align*}
    \expect{\avgloss(\model{}{T}) - \optloss}   \leq \frac{7}{6} \left(\avgloss(\model{}{0}) - \optloss  \right) \cdot \exp(-\frac{\mu T}{108 L})+ \frac{4374L \left(\lambda+\frac{1}{n-f} \right) \sigma^2}{T\mu^2}+ \frac{9 \lambda\heter^2}{2\mu} \enspace.
\end{align*}
This is the desired result.
\end{proof}
\subsection{Proof of Corollary~\ref{cor:main}}

As $n \geq (2+\nu) f$, we have

\begin{align*}
    \frac{2}{2+\nu} n \geq 2f\enspace.
\end{align*}
Rearranging the terms we have
\begin{align*}
    n - 2f \geq \left(1-\frac{2}{2+\nu} \right)n = \frac{\nu}{2+\nu} n\enspace.
\end{align*}
Therefore,
\begin{align*}
    \frac{f}{n-2f} \leq \frac{2+\nu}{\nu}\cdot \frac{f}{n}\enspace.
\end{align*}
As $\nu > 0$ is a constant, we have
\begin{align}
\label{eq:lambdaroder}
    \lambda = \frac{6 f}{n-2f}\, \left( 1 + \frac{f}{n-2f} \right) \leq \frac{2+\nu}{\nu}\cdot \frac{6f}{n} \left( 1 + \frac{2+\nu}{\nu}\cdot \frac{f}{n}\right) \in \mathcal{O}\left(\frac{f}{n} \right)\enspace.
\end{align}
Theorem~\ref{thm:main} then implies that 
\begin{align*}
        \avgloss(\model{}{T}) - \optloss   \in  \mathcal{O} \left( Q_0\cdot \exp(-\frac{\mu T}{108 L})+\frac{L \left(\lambda+\frac{1}{n-f} \right) \sigma^2}{T\mu^2}+ \frac{ \lambda\heter^2}{\mu} \right)\enspace,
\end{align*}
Combining this with~\eqref{eq:lambdaroder}, and noting that $\frac{1}{n-f}\leq \frac{2}{n}$, we have
\begin{align}\label{eq:sssvbv}
        \avgloss(\model{}{T}) - \optloss   \in  \mathcal{O}\left(Q_0\cdot \exp(-\frac{\mu T}{108 L}) + \frac{L \sigma^2}{\mu^2 T}\left(\frac{1}{n}+\frac{f}{n}\right) + \frac{f}{n} \frac{\heter^2}{\mu} \right)\enspace.
    \end{align}
    Now note that as $T \rightarrow \infty$, the first two terms converge to $0$. More precisely, for any $\varepsilon > 0$, setting $$T = \max \left\{\frac{2L \sigma^2}{\mu^2 \varepsilon}\left(\frac{f+1}{n}\right), 108 \frac{L}{\mu} \log\frac{2Q_0}{\varepsilon} \right\} \leq \frac{2L \sigma^2}{\mu^2 \varepsilon}\left(\frac{f+1}{n}\right)+ 108 \frac{L}{\mu} \log\frac{2Q_0}{\varepsilon}  ,$$ we obtain that
    \begin{align*}
        Q_0\cdot \exp(-\frac{\mu T}{108 L}) + \frac{L \sigma^2}{\mu^2 T}\left(\frac{1}{n}+\frac{f}{n}\right) \leq \varepsilon.
    \end{align*}
    Combing this with~\eqref{eq:sssvbv}, we have 
    $$\avgloss(\model{}{T}) - \optloss   \in  \mathcal{O}\left(\frac{f}{n} \cdot \frac{\heter^2}{\mu} + \varepsilon\right)\enspace,$$
    for $$T \in \mathcal{O}\left(\frac{L \sigma^2}{\mu^2 \varepsilon}\left(\frac{f+1}{n}\right)+ \frac{L}{\mu} \log\frac{Q_0}{\varepsilon}  \right),$$
    which is the desired result.

\section{Proof of Theorem~\ref{thm:upperboundgd}}
\label{app:proofupperboundgd}

Let us denote $R_t \coloneqq \text{TM}^{(f)}\left(G_t^{(1)}, \ldots, \, G_t^{(n)} \right)$ and $\bar{G}_t := \sum_{i \in \H} G_t^{(i)}$. By Proposition 2 in~\citep{allouah2023fixing}, we have
\begin{align}
\label{eq:first12}
    \norm{R_t -\bar{G}_t}^2  &\leq \lambda\frac{1}{n-f} \sum_{i \in \H} \norm{{G}_t^{(i)} -\bar{G}_t}^2, \quad \text{where} \quad \lambda = \frac{6 f}{n-2f}\, \left( 1 + \frac{f}{n-2f} \right) \enspace.
\end{align}
Similarly, for each $i \in \H$ and $\theta_t \in \R^d$, we have
\begin{align*}
    \norm{{G}_t^{(i)}  - \nabla  \localloss{i}(\model{}{t})}^2  &\leq \lambda'\frac{1}{m-b} \sum_{j \in \SSS_h^{(i)}} \norm{\nabla q(x^{(i,j)}, \theta_t) -\nabla  \localloss{i}(\model{}{t})}^2 , 
\end{align*}
where $\lambda' = \frac{6 b}{m-2b}\, \left( 1 + \frac{m}{m-2b} \right)$ . Therefore, by Assumption~\ref{asp:bnd_var}, we have
\begin{align}
\label{eq:second12}
    \norm{{G}_t^{(i)}  - \nabla  \localloss{i}(\model{}{t})}^2  &\leq \lambda'\sigma^2.
\end{align}
We now prove a few useful lemmas.
\begin{lemma}
\label{lem:decsent_gd}
Suppose Assumption~\ref{asp:lip}. Consider Algorithm~\ref{algorithm:d-gd} with $T\geq 2$, and $\gamma \leq \nicefrac{1}{L}$. Then, for all $t \in \{0,\dots,T-1\}$, the following holds true:
\begin{align*}
\avgloss(\model{}{t+1}) - \avgloss(\model{}{t})  &\leq -\frac{\gamma}{2}\norm{\nabla  \avgloss(\model{}{t})}^2 +\frac{\gamma}{2}\norm{R_t - \nabla  \avgloss(\model{}{t})}^2.
\end{align*}
\end{lemma}
\begin{proof}
    Consider an arbitrary step $t$.  Note that Assumption~\ref{asp:lip} implies $L$-Lipschitz continuity of $\nabla \avgloss(\model{}{})$. Thus, we have 
\begin{align*}
    \avgloss(\model{}{t+1}) - \avgloss(\model{}{t})  &\leq \iprod{\model{}{t+1} - \model{}{t}}{\nabla  \avgloss(\model{}{t})} + \frac{L}{2} \norm{\model{}{t+1} - \model{}{t}}^2\enspace.
\end{align*}
Substituting from Algorithm~\ref{algorithm:d-gd}, $\model{}{t+1} = \model{}{t} - \gamma R_t$, we obtain that
\begin{align*}
    \avgloss(\model{}{t+1}) - \avgloss(\model{}{t})  \leq -\gamma\iprod{R_t}{\nabla  \avgloss(\model{}{t})} + \frac{L\gamma^2}{2} \norm{R_t}^2\enspace.
\end{align*}
Using the fact that $2\iprod{a}{b} = \norm{a}^2 + \norm{b}^2- \norm{a-b}^2$, we obtain that
\begin{align*}
    \avgloss(\model{}{t+1}) - \avgloss(\model{}{t})  &\leq -\frac{\gamma}{2}\norm{R_t}^2-\frac{\gamma}{2}\norm{\nabla  \avgloss(\model{}{t})}^2 +\frac{\gamma}{2}\norm{R_t - \nabla  \avgloss(\model{}{t})}^2 + \frac{L\gamma^2}{2} \norm{R_t}^2\\
    &= \left(\frac{L\gamma^2}{2}-\frac{\gamma}{2}\right) \norm{R_t}^2-\frac{\gamma}{2}\norm{\nabla  \avgloss(\model{}{t})}^2 +\frac{\gamma}{2}\norm{R_t - \nabla  \avgloss(\model{}{t})}^2\enspace.
\end{align*}
As $\gamma \leq \frac{1}{L}$, we have $\left(\frac{L\gamma^2}{2}-\frac{\gamma}{2}\right) \leq 0$ in the above, thereby proving the lemma.
\end{proof}

\begin{lemma}
\label{lemma:boundederrorgd}
    Suppose assumptions~\ref{asp:bnd_var}, and~\ref{asp:heter} hold true. Consider Algorithm~\ref{algorithm:d-gd}.
    For all $t \in \{0,\dots,T-1\}$, the following holds true:
    \begin{align*}
    \norm{R_t - \nabla  \avgloss(\model{}{t})}^2 \leq 2 \lambda' \sigma^2  + 6  \lambda\lambda' \sigma^2 + 6 \lambda\zeta^2.
\end{align*}
\end{lemma}
\begin{proof}
    From the triangle and the Jensen's inequalities, we obtain that 
\begin{align*}
    \norm{R_t - \nabla  \avgloss(\model{}{t})}^2 =  \norm{R_t -\bar{G}_t  + \bar{G}_t  - \nabla  \avgloss(\model{}{t})}^2 \leq 2 \norm{R_t -\bar{G}_t}^2 + 2 \norm{\bar{G}_t  - \nabla  \avgloss(\model{}{t})}^2 .
\end{align*}
From Jensen's inequality, we have
\begin{align*}
    \norm{\bar{G}_t  - \nabla  \avgloss(\model{}{t})}^2 &= \norm{\frac{1}{n-f}\sum_{i \in \H}({G}_t^{(i)}  - \nabla  \localloss{i}(\model{}{t}))}^2 \leq \frac{1}{n-f} \sum_{i \in \H}\norm{{G}_t^{(i)}  - \nabla  \localloss{i}(\model{}{t})}^2 \leq \lambda' \sigma^2 .
\end{align*}
Moreover, we have
\begin{align*}
    & \norm{R_t -\bar{G}_t}^2  \leq \lambda\frac{1}{n-f} \sum_{i \in \H} \norm{{G}_t^{(i)} -\bar{G}_t}^2 \\
    &= \lambda\frac{1}{2(n-f)^2} \sum_{i,j \in \H} \norm{{G}_t^{(i)} -{G}_t^{(j)} }^2 \\
    &\leq  \lambda\frac{1}{2(n-f)^2} \sum_{i,j \in \H}\norm{{G}_t^{(i)} - \nabla\localloss{i}(\model{}{t})+ \nabla\localloss{i}(\model{}{t}) -\nabla\localloss{j}(\model{}{t}) + \nabla\localloss{j}(\model{}{t}) -{G}_t^{(j)} }^2 \\
    &\leq \lambda\frac{3}{2(n-f)^2} \sum_{i,j \in \H}\left(\norm{{G}_t^{(i)} - \nabla\localloss{i}(\model{}{t})}^2 +\norm{ \nabla\localloss{i}(\model{}{t}) -\nabla\localloss{j}(\model{}{t})}^2  +\norm{ \nabla\localloss{j}(\model{}{t}) -{G}_t^{(j)} }^2 \right)\\
    &= \lambda\frac{3}{n-f} \sum_{i \in \H}\norm{{G}_t^{(i)} - \nabla\localloss{i}(\model{}{t})}^2 + \lambda\frac{3}{2(n-f)^2} \sum_{i,j \in \H}\norm{ \nabla\localloss{i}(\model{}{t}) -\nabla\localloss{j}(\model{}{t})}^2 .
\end{align*}
From~\eqref{eq:second12}, for all $i \in \H$, we have $\norm{{G}_t^{(i)} - \nabla\localloss{i}(\model{}{t})}^2 \leq \lambda' \sigma^2$. Furthermore, by Assumption~\ref{asp:heter}, $\frac{1}{2(n-f)^2}\sum_{i,j \in \H}\norm{ \nabla\localloss{i}(\model{}{t}) -\nabla\localloss{j}(\model{}{t})}^2 = \frac{1}{n-f} \sum_{i \in \H}\norm{ \nabla\localloss{i}(\model{}{t}) -\nabla  \avgloss(\model{}{t})}^2 \leq \zeta^2$. Thus, from above we obtain that
\begin{align*}
    \norm{R_t -\bar{G}_t}^2  \leq 3  \lambda\lambda' \sigma^2 + 3 \lambda\zeta^2 .
\end{align*}
Combining the above we obtain that 
\begin{align*}
    \norm{R_t - \nabla  \avgloss(\model{}{t})}^2 \leq 2 \lambda' \sigma^2  + 6  \lambda\lambda' \sigma^2 + 6 \lambda\zeta^2 .
\end{align*}
\end{proof}

\begin{proof}[Back to the proof of Theorem~\ref{thm:upperboundgd}]
    Using the fact that the loss function satisfies the PL condition, from Lemma~\ref{lem:decsent_gd}, we obtain that 
    \begin{align*}
        \avgloss(\model{}{t+1}) - \avgloss(\model{}{t})  &\leq -\frac{\gamma}{2}\norm{\nabla  \avgloss(\model{}{t})}^2 +\frac{\gamma}{2}\norm{R_t - \nabla  \avgloss(\model{}{t})}^2\\
        &\leq -{\mu \gamma} (\avgloss(\model{}{t}) - Q^*)+\frac{\gamma}{2}\norm{R_t - \nabla  \avgloss(\model{}{t})}^2.
    \end{align*}
    Therefore, substituting from Lemma~\ref{lemma:boundederrorgd} in the above, we obtain that
    \begin{align*}
        \avgloss(\model{}{t+1}) - Q^* &\leq (1-{\mu \gamma}) (\avgloss(\model{}{t}) - Q^*)+\frac{\gamma}{2}\norm{R_t - \nabla  \avgloss(\model{}{t})}^2\\
        & \leq  (1-{\mu \gamma}) (\avgloss(\model{}{t}) - Q^*)+{\gamma}(  \lambda' \sigma^2  + 3 \lambda\lambda' \sigma^2 + 3 \lambda\zeta^2) .
    \end{align*}
    Recall that the above holds true for any $t \in \{0, \ldots, T -1 \}$. As $\mu \leq L$, we have $1 - \mu \gamma = 1 - \frac{\mu}{L} \in [0 , 1)$. Thus, substituting $\gamma = 1/L$ and applying the inequality recursively, we obtain that
    \begin{align*}
        \avgloss(\model{}{T}) - Q^* &\leq \left( 1 - \frac{\mu}{L}\right)^T \left( \avgloss(\model{}{0}) - Q^* \right) + \frac{1}{\mu}(  \lambda' \sigma^2  + 3 \lambda\lambda' \sigma^2 + 3 \lambda\zeta^2) .
    \end{align*}
    As $\left( 1 - \frac{\mu}{L}\right)^T \leq \exp \left( -\frac{\mu}{L}T\right)$, the above proves the theorem. 
\end{proof}

\end{document}